%% file: camera_ready.tex
\documentclass{article}

\usepackage{microtype}
\usepackage{graphicx}
\usepackage{subfigure}
\usepackage{booktabs} 

\usepackage[accepted]{icml2024}

\usepackage{amsmath}
\usepackage{amssymb}
\usepackage{mathtools}
\usepackage{amsthm}

\usepackage{hyperref}
\usepackage{url}
\usepackage{tikz} 
\usepackage{wrapfig}
\usepackage{colortbl}
\usepackage[capitalize,noabbrev]{cleveref}

\theoremstyle{plain}
\newtheorem{theorem}{Theorem}[section]
\newtheorem{proposition}[theorem]{Proposition}
\newtheorem{lemma}[theorem]{Lemma}

\theoremstyle{definition}

\theoremstyle{remark}

\usepackage[textsize=tiny]{todonotes}
\input{math_commands.tex}

\icmltitlerunning{D-Flow: Differentiating through Flows for Controlled Generation}

\begin{document}

\twocolumn[
\icmltitle{D-Flow: Differentiating through Flows for Controlled Generation}



\icmlsetsymbol{equal}{*}

\begin{icmlauthorlist}
\icmlauthor{Heli Ben-Hamu}{zzz,equal}
\icmlauthor{Omri Puny}{zzz}
\icmlauthor{Itai Gat}{yyy}
\icmlauthor{Brian Karrer}{yyy}
\icmlauthor{Uriel Singer}{yyy}
\icmlauthor{Yaron Lipman}{yyy,zzz}
\end{icmlauthorlist}

\icmlaffiliation{zzz}{Weizmann Institute of Science}
\icmlaffiliation{yyy}{Meta}

\icmlcorrespondingauthor{Heli Ben-Hamu}{heli.benhamu@weizmann.ac.il}

\icmlkeywords{Machine Learning, ICML}
\vskip 0.3in
]



\printAffiliationsAndNotice{\textsuperscript{*}Work done as a research intern at FAIR, Meta.}  

\begin{abstract}
Taming the generation outcome of state of the art Diffusion and Flow-Matching (FM) models without having to re-train a task-specific model unlocks a powerful tool for solving inverse problems, conditional generation, and controlled generation in general. In this work we introduce \emph{D-Flow}, a simple framework for controlling the generation process by differentiating through the flow, optimizing for the source (noise) point. We motivate this framework by our key observation stating that for Diffusion/FM models trained with Gaussian probability paths, differentiating through the generation process projects gradient on the data manifold, implicitly injecting the prior into the optimization process. We validate our framework on linear and non-linear controlled generation problems including: image and audio inverse problems and conditional molecule generation reaching state of the art performance across all. 
\end{abstract}

\section{Introduction}

Controlled generation from generative priors is of great interest in many domains. Various problems such as conditional generation, inverse problems, sample editing etc., can all be framed as a controlled generation problem. In this work we focus on controlled generation from diffusion/flow generative models \cite{song2019generative,ho2020denoising,lipman2023flow} as they are the current state-of-the-art generative approaches across different data modalities.

There are three main approaches for controlled generation from diffusion/flow models: (i) conditional  training, where the model receives the condition as an additional input during training \cite{song2020score,dhariwal2021diffusion,ho2022classifierfree}, although performing very well this approach requires task specific training of a generative model which in cases may be prohibitive; (ii) training-free approaches that modify the generation process of a pre-trained model, adding additional guidance \cite{bar2023multidiffusion,yu2023freedom}. The guidance is usually built upon strong assumptions on the generation process that can lead to errors in the generation and mostly limit the method to observations that are linear in the target \cite{kawar2022denoising,chung2022improving,song2023pseudoinverseguided,pokle2023trainingfree}; lastly, (iii) adopt a variational perspective, framing the controlled generation as an optimization problem \cite{graikos2023diffusion,mardani2023variational,wallace2023endtoend,samuel2023generating}, requiring only a differentiable cost to enforce the control. This paper belongs to this third class. 

\begin{figure}
  \begin{center}
  \begin{tabular}{@{\hspace{0pt}}c@{\hspace{3pt}}c@{\hspace{3pt}}c@{\hspace{8pt}}c@{\hspace{0pt}}}
       \includegraphics[width=0.22\columnwidth]{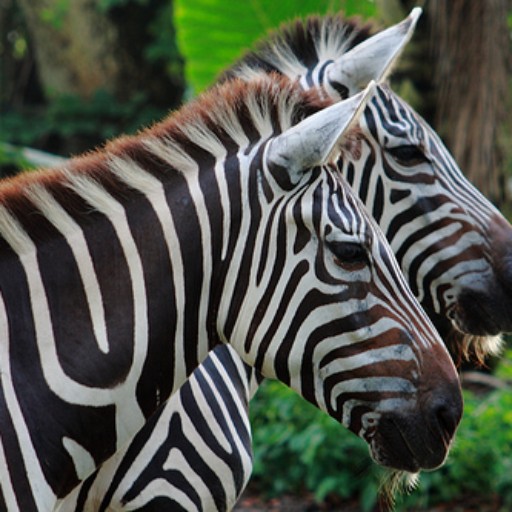} &  
       \includegraphics[width=0.22\columnwidth]{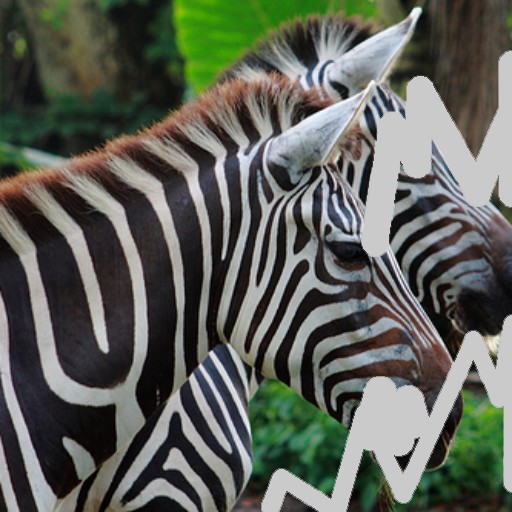} &
       \includegraphics[width=0.22\columnwidth]{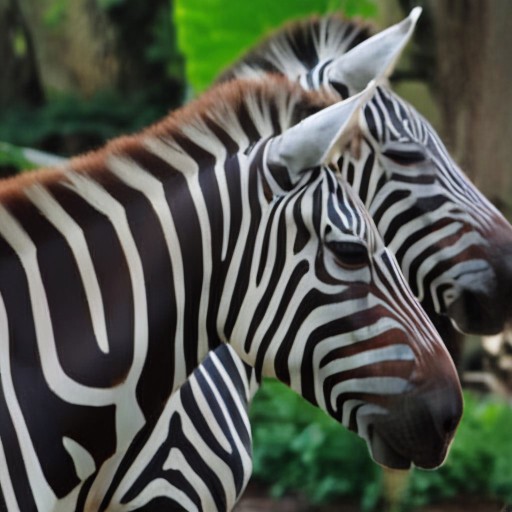} &
       \includegraphics[width=0.22\columnwidth]{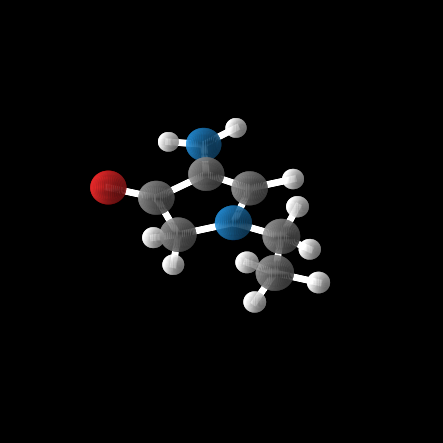}\\ 
       \scriptsize GT &  \scriptsize Distorted & \scriptsize Ours & \scriptsize $\alpha = 96.62$   
        
  \end{tabular} \vspace{10pt}
  
  \begin{tabular}{@{\hspace{0pt}}c@{\hspace{3pt}}c@{\hspace{3pt}}c@{\hspace{0pt}}}
         \includegraphics[trim={10cm 0pt 9.8cm 0pt},clip,width=0.31\columnwidth]{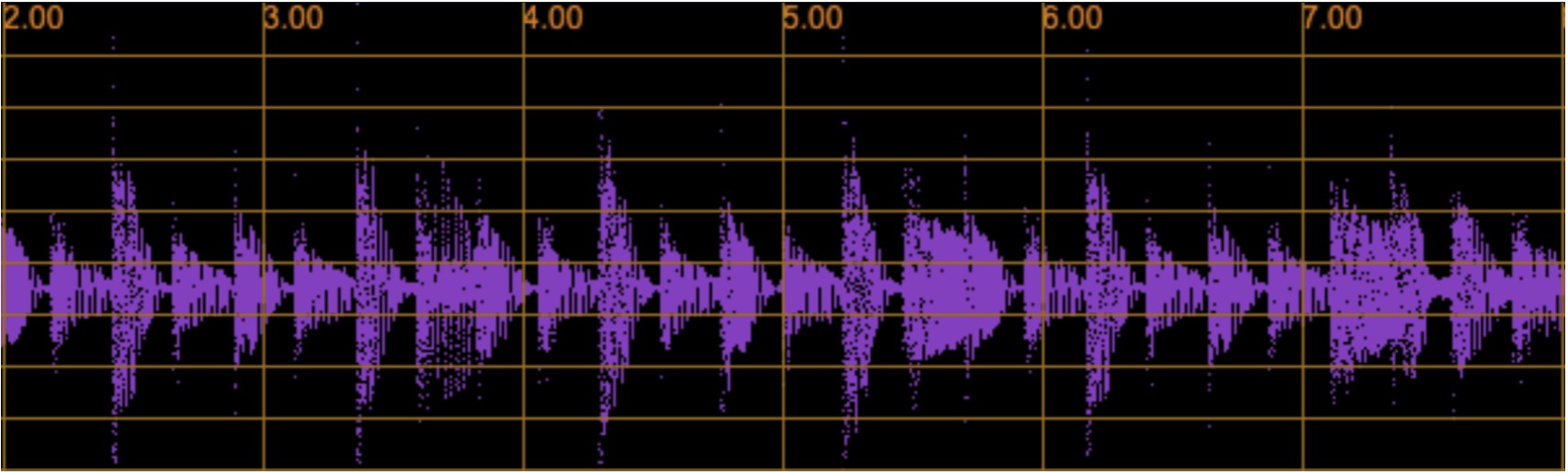} &
         \includegraphics[trim={10cm 0pt 9.8cm 0pt},clip,width=0.31\columnwidth]{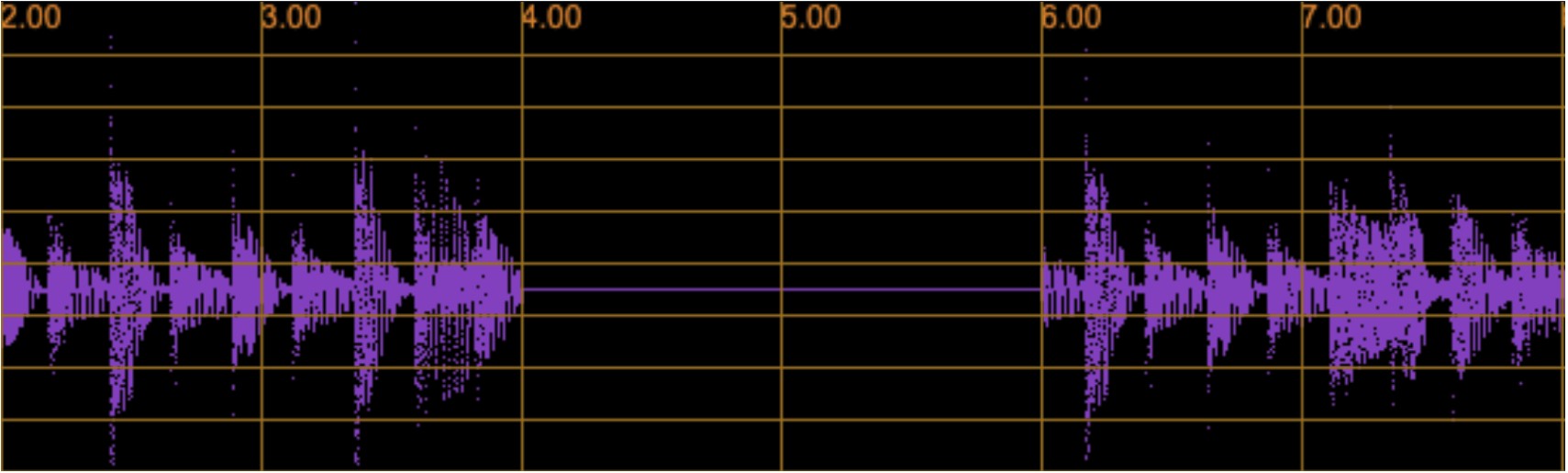} &
         \includegraphics[trim={10cm 0pt 9.8cm 0pt},clip,width=0.31\columnwidth]{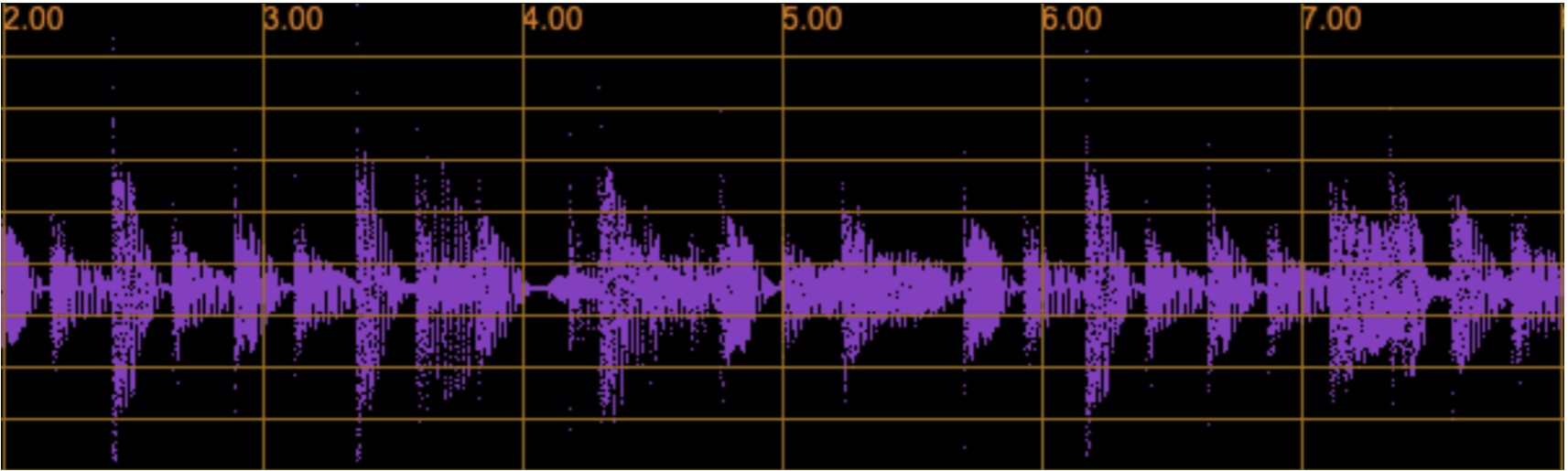}\\
       \scriptsize GT &  \scriptsize Distorted &\scriptsize  Ours
        
  \end{tabular}

  \end{center}
  \caption{Free-form inpainting with a latent T2I FM model (Ground truth image is taken from the MS-COCO validation set), conditionally generated molecule and audio inpainting using D-Flow.}\label{fig:brush_inpainting}
  \vspace{-10pt}
\end{figure}

The goal of this paper is to introduce a framework for adding controlled generation to a pre-trained Diffusion or Flow-Matching (FM) model based on \emph{differentiation through the ODE sampling process}. Our key observation is that for Diffusion/FM models trained with standard Gaussian probability paths, differentiating an arbitrary loss $\gL(x)$ through the generation process of $x$ with respect to the initial point, $x_0$, projects the gradient $\nabla_x\gL$ onto the ``data manifold'', \ie, onto major data directions at $x$, implicitly injecting a valuable prior. Based on this observation we advocate a simple general algorithm that minimizes an arbitrary cost function $\gL(x)$, representing the desired control, as a function of the source noise point $x_0$ used to generate $x$. That is,
\begin{equation}\label{e:basic}
    \min_{x_0}\ \gL(x).
\end{equation}

Differentiating through a generator of a GAN or a normalizing flow was proven generally useful for controlled generation \cite{bora2017compressed,asim2020invertible,whang2021solving} and counterfactual examples \cite{dombrowski2021diffeomorphic,dombrowski2024diffeomorphic_b}. Recently, \cite{wallace2023endtoend,samuel2023generating} have been suggesting to differentiate through a discrete diffusion  solver for the particular tasks of incorporating classifier guidance and generating rare concepts. In this paper we generalize this idea in two ways: (i) we consider general flow models trained with Gaussian probability paths, including Diffusion and Flow-Matching models; and (ii) we demonstrate, both theoretically and practically, that the inductive bias injected by differentiating through the flow is applicable to a much wider class of problems modeled by general cost functions. 

We experiment with our method on a variety of settings and applications: Inverse problems on images using conditional ImageNet and text-2-image (T2I) generative priors, conditional molecule generation with QM9 unconditional generative priors, and audio inpainting and super-resolution with unconditional generative prior. In all application we were able to achieve state of the art performance without carefully tuning the algorithm across domains and applications. One drawback of our method is the relative long time for generation (usually $5-15$ minutes on ImageNet-128 on an NVidia V100 GPU) compared to some baselines, however the method's simplicity and its superior results can justify its usage and adaptation in many use cases. Furthermore, we believe there is great room for speed improvement. 

To summarize, our contributions are:
\begin{itemize}
    \item We formulate the controlled generation problem as a simple source point optimization problem using general flow generative models.
    \item We show that source point optimization of flows trained with Gaussian probability paths inject an implicit bias exhibiting a data-manifold projection behaviour to the cost function's gradient. 
    \item We empirically show the generality and the effectiveness of the proposed approach for different domains. 
\end{itemize}

\begin{figure*}
  \begin{center}
  \begin{tabular}{@{\hspace{0pt}}c@{\hspace{4pt}}c@{\hspace{4pt}}c@{\hspace{4pt}}c@{\hspace{4pt}}c@{\hspace{4pt}}c@{\hspace{4pt}}c@{\hspace{4pt}}c@{\hspace{4pt}}c@{\hspace{0pt}}}
       \includegraphics[width=0.2\columnwidth]{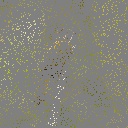} &  
       \includegraphics[width=0.2\columnwidth]{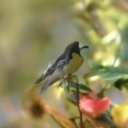} &
       \includegraphics[width=0.2\columnwidth]{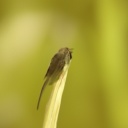} & 
       \includegraphics[width=0.2\columnwidth]{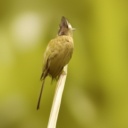} &
       \includegraphics[width=0.2\columnwidth]{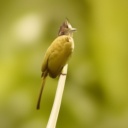} &
       \includegraphics[width=0.2\columnwidth]{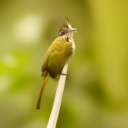} &
       \includegraphics[width=0.2\columnwidth]{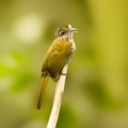} &
       \includegraphics[width=0.2\columnwidth]{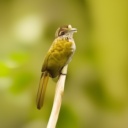} &
       \includegraphics[width=0.2\columnwidth]{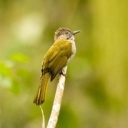} \\
       Distorted & Initial $x(1)$ & step 2 & step 4 &  step 6 &  step 8 & step 10 & step 12 & GT \vspace{-10pt}
  \end{tabular}    
  \end{center}
  \caption{Intermediate $x(1)$ during optimization. Given a distorted image and randomly initialized $x_0$ defining the initial $x(1)$, our optimization travels close to the natural image manifold passing through in-distribution images on its way to the GT sample from the face-blurred ImageNet-128 validation set.}\label{fig:implicit_bias_optimization}
\end{figure*}

\begin{algorithm}[t]
\caption{D-Flow.}\label{alg:main}
\begin{algorithmic}
\STATE {\bfseries Require:}  cost $\gL$, pre-trained flow model $u_t(x)$

\STATE Initialize $x_0^{(0)}=x_0$ 
\FOR{$i= 1, \dots, N$} 
    \STATE $x^{(i)}(1) \gets \texttt{solve}(x_0^{(i)},u_t)$ 
    \STATE $x_0^{(i+1)} \gets \texttt{optimize\_step}(x_0^{(i)},\nabla_{x_0}\gL(x^{(i)}(1)))$ 
\ENDFOR
\STATE {\bfseries Return:} $x^{N}(1)$
\end{algorithmic}
\end{algorithm}

\section{Preliminaries}
\textbf{Flow models. } Generative flow models, including Continuous Normalizing Flows (CNFs) \cite{chen2018neural,lipman2023flow} and (deterministic sampling of) Diffusion Models \cite{song2020score} generate samples $x(1)\in\Real^d$ by first sampling from some source (noise) distribution $x(0)\sim p_0(x_0)$ and then solving an Ordinary Differential Equation (ODE),
\begin{align}\label{e:dynamics}
    \dot{x}(t) = u_t(x(t)),
\end{align}
from time $t=0$ to time $t=1$\footnote{In this paper we use the convention of $t=0$ corresponds to noise, and $t=1$ to data.}, using a predetermined velocity field $u:[0,1]\times \Real^d\too\Real^d$. We denote by $p_1$ the distribution and density function of $x(1)$ given $x(0)\sim p_0(x_0)$.

\section{Controlled Generation via Source Point Optimization}

Given a pre-trained (frozen) flow model, $u_t(x)$, represented by a neural network and some cost function $\gL:\Real^d\too\Real_+$, our goal is to find likely samples $x$ that provide low cost $\gL(x)$ and are likely under the flow model's distribution $p_1$. We advocate a general framework formulating this problem as the following optimization problem 
\begin{equation}\label{e:method_inverse_sol}
    \min_{x_0} \quad \gL(x(1))     
\end{equation}
where in general $\gL$ can also incorporate multiple costs including potentially a regularization term that can depend on $x_0$ and $u$, 
\begin{equation}\label{e:L_with_R}
    \tilde{\gL}(x) = \gL(x) +  \gR(x_0, u). 
\end{equation}
In this formulation, the sample $x(1)$ is constrained to be a solution of the ODE \ref{e:dynamics} with initial boundary condition $x(0)=x_0$, where $x_0$ is the only optimized quantity and $\gL$ is the desired cost function. Optimizing \eqref{e:method_inverse_sol} is done by computing the gradients of the loss w.r.t.~the optimized variable $x_0$ as listed in Algorithm \ref{alg:main}. We call this method \emph{D-Flow}. To better understand the generality of this framework we next consider several instantiations of \eqref{e:method_inverse_sol}.

\subsection{Cost Functions}

\paragraph{Reversed sampling.} First, consider the simple case where $\gL(x)= \norm{x-y}^2$. In this case, the solution of \ref{e:method_inverse_sol} will be the $x_0$ that that has an ODE trajectory that reaches $y$ at $t=1$, \ie, $x(1)=y$. Note that since (under some mild assumptions on $u_t(x)$) \eqref{e:dynamics} defines a diffeomorphism $\Real^d\too\Real^d$, for an arbitrary $y\in\Real^d$, there exists a unique solution $x_0\in\Real^d$ to \eqref{e:method_inverse_sol}. 

\paragraph{Inverse problems.} In this case we have access to some known corruption function $H:\Real^d\too \Real^n$ and a corrupted sample from an unknown ground truth signal $x_*$,
\begin{equation} \label{e:inverse_problem}
    y = H(x_*) + \eps,
\end{equation}
where $\eps \sim \gN(\eps)$ is an optional additive noise. The goal is to recover an $x$ that produces $y$ and the cost function is usually
\begin{equation}
    \gL(x) = \norm{H(x)-y}^2,
\end{equation}
where the norm can be some arbitrary $L_p$ norm or even a general loss $\ell(H(x),y)$ comparing $H(x)$ and $y$. Specific choices of the corruption function $H$ can lead to common applications: \emph{Image inpainting} corresponds to choosing the corruption function $H$ to sub-sample known $n<d$ pixels out of $d$ total pixels; \emph{Image deblurring} corresponds to taking $H:\Real^d\too\Real^d$ to be a blurring function, \eg, a convolution with a blurring kernel; \emph{Super-resolution} corresponds to $H:\Real^d\too \Real^{d/k}$ lowering the dimension by a factor of $k$. 

\paragraph{Conditional sampling.} Another important application is to guide the sampling process to satisfy some conditioning $y$. In this case we can take $\gL(x)$ to encourage a classifier or some energy function to reach a particular class or energy $y$. For example, if $\gF:\Real^d\too\Real$ is some function and we would like to generate a sample from a certain level set $c\in\Real$ we can use the loss
\begin{equation}\label{eq:cond_sample}
    \gL(x) = \parr{\gF(x) - c}^2.
\end{equation}

\subsection{Initialization}
The initialization of $x_0$ can have a great impact on the convergence of the optimization of \eqref{e:method_inverse_sol}. A natural choice will be to initialize $x_0$ with a sample from the source distribution $p_0(x_0)$. We find that for cases when an observed signal $y$ provides a lot of information about the desired $x$, one can improve the convergence speed of the optimization. For example, in linear inverse problems on images, where the observed $y$ has a strong prior on the structure of the image, it is beneficial to initialize $x_0$ with a blend of a sample from the source distribution and the backward solution of the ODE from $t=1$ to $t=0$ of $y$:
\begin{equation}
    x_0 = \sqrt{\alpha}\cdot y(0) + \sqrt{1-\alpha}\cdot z,
\end{equation}
where $z\sim p_0(x_0)$ and $y(0)=y + \int_1^0 u(t,y(t))dt$.

\subsection{Regularizations}
The formulation in \eqref{e:method_inverse_sol} allows including different regularizations $\gR$ (\eqref{e:L_with_R}) discussed next. Maybe the most intriguing of these regularizations, and the main point of this paper, is the \emph{implicit regularization}, \ie, corresponding to $\gR\equiv 0$, discussed last in what follows. 

\paragraph{Regularizing the target $x(1)$.} Maybe the most natural is incorporating the negative log likelihood (NLL) of the sample $x(1)$, \ie, $\gR=-\log p_1(x(1))$ in \eqref{e:L_with_R}. This prior can be incorporated by augmenting $x(t)\in \Real^d$ with an extra coordinate $z\in \Real$ and formulate \eqref{e:method_inverse_sol} as
\begin{subequations}
\begin{align}
    \min_{x_0} & \quad \gL(x(1)) - z(1)\\ \label{e:log_p_dot_x}
    \text{s.t.} & \quad \dot{x}(t) = u_t(x(t)),\quad \quad \  \ x(0)=x_0 \\ \label{e:log_p_dot_z}
    & \quad \dot{z}(t) = -\divv\, u_t(x(t)), \ z(0)=\log p_0(x_0) 
\end{align}    
\end{subequations}
\begin{wrapfigure}{r}{0.45\columnwidth}
  \begin{center}
  \begin{tabular}{@{\hspace{0pt}}c@{\hspace{3pt}}c@{\hspace{0pt}}}
       \includegraphics[width=0.18\columnwidth]{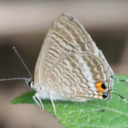} &  
       \includegraphics[width=0.18\columnwidth]{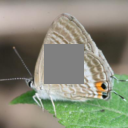} \\
       ${\scriptstyle \text{BPD}=\text{2.02}}$ & ${\scriptstyle \text{BPD}=\text{1.84}}$\vspace{-10pt}
  \end{tabular}    
  \end{center}
  \caption{BPD of two images in an ImageNet-128 model.}\label{fig:likelihoods}
\end{wrapfigure}
Indeed, solving the ODE system defined by equations \ref{e:log_p_dot_x} and \ref{e:log_p_dot_z} for times $t\in[0,1]$ provides $z(1)=\log p_1(x(1))$, see \cite{chen2018neural}. However, aside from the extra complexity introduced by the divergence term in the ODE in \eqref {e:log_p_dot_z} (see \eg, \cite{grathwohl2018ffjord} for ways to deal with this type of ODE) it is not clear whether likelihood is a good prior in deep generative models in high dimensions \cite{nalisnick2019deep}; In Figure \ref{fig:likelihoods} we compare bits-per-dimension (BPD) of a test image of ImageNet-128 and a version of this image with a middle square masked with zeros, providing a more likely image according to our flow model trained on ImageNet.  

\paragraph{Regularizing the source $x(0)=x_0$.} Another option is to regularize the source point $x(0)=x_0$. The first choice would again be to incorporate the NLL of the noise sample, \ie, $\gR=-\log p_0(x_0)$, which for standard noise $p_0(x_0)=\gN(x_0|0,I)$ would reduce to $\gR=c+\frac{1}{2}\norm{x_0}^2 $, where $c$ is a constant independent of $x_0$. This however, would attract $x_0$ towards the most likely all zero mean but far from most of the probability mass at norm $\sqrt{d}$. 

Following \cite{samuel2023normguided} we instead prefer to make sure $x_0$ stays in the area where most mass of $p_0$ is concentrated and therefore use the $\chi^d$ distribution, which is defined as the probability distribution $p(r)$ of the random variable $r =\norm{x_0}$ where $x_0\sim \gN(x_0|0,I)$ is again the standard normal distribution. The NLL of $r$ in this case is
\begin{equation}\label{e:chid_reg}
    \gR = -\log p(r) = c + (d-1)\log \norm{x_0} - \frac{\norm{x_0}^2}{2},
\end{equation}
where $c$ is a constant independent of $x_0$.

\paragraph{Implicit regularization.} Maybe the most interesting and potentially useful regularization in our formulation (\eqref{e:method_inverse_sol}) comes from the choice of optimizing the cost $\gL(x(1))$ as a function of the source point $x(0)=x_0$. For standard diffusion/flow models that are trained to zero loss:
\definecolor{mygray}{gray}{0.95}
\begin{center}\vspace{-17pt}			
    \colorbox{mygray} {		
      \begin{minipage}{0.977\linewidth} 	
Optimizing the cost $\gL(x(1))$ with respect to $x_0$ follows the data distribution $p_1(x_1)$ by projecting the gradient $\nabla_{x(1)}\gL(x(1))$ with the local data covariance matrix. 
      \end{minipage}}			
      \vspace{-1em}
\end{center}

\begin{wrapfigure}{r}{0.38\columnwidth}
  \begin{center}
  \includegraphics[width=0.38\columnwidth]{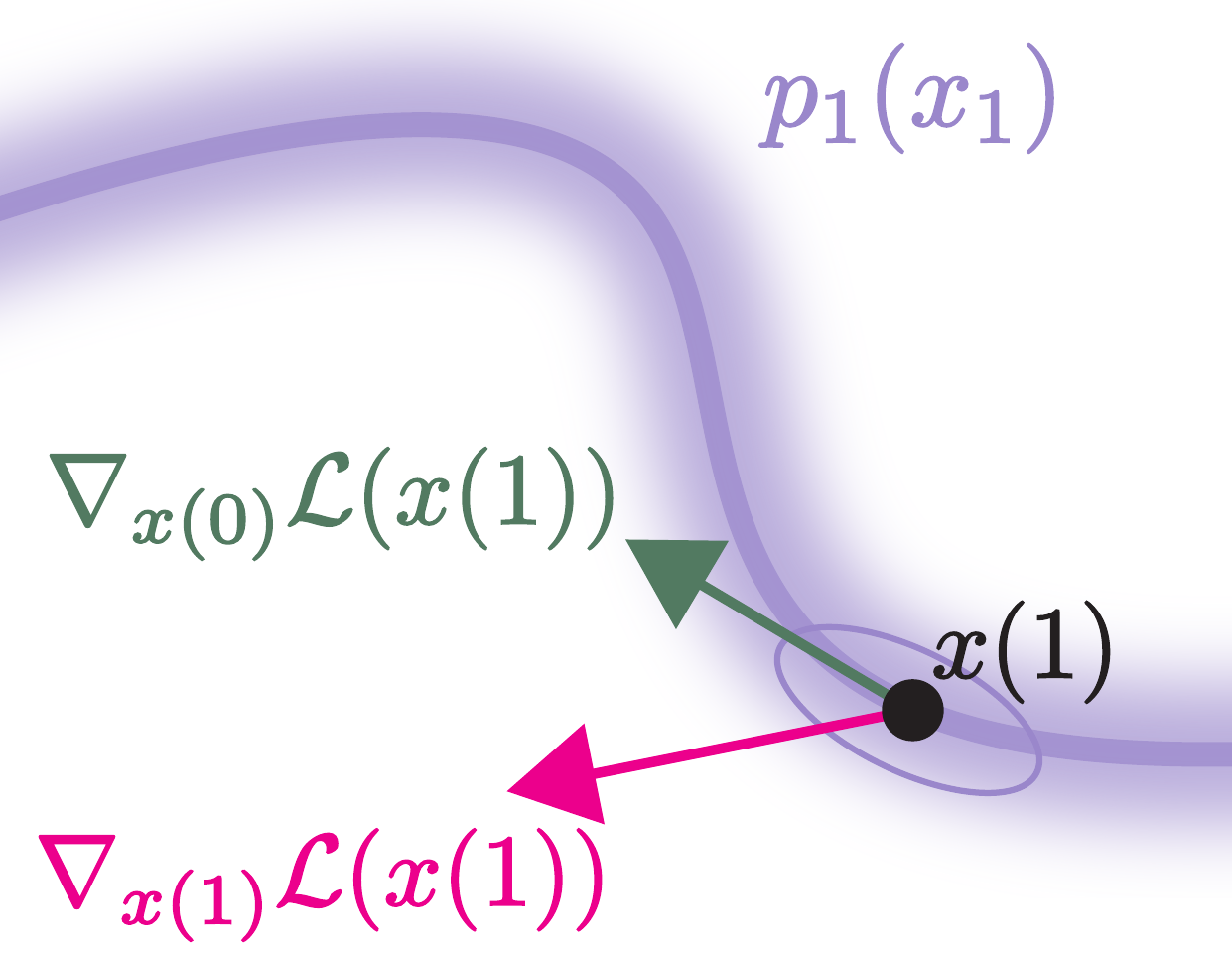} \vspace{-20pt}
  \end{center}
  \caption{Implicit bias in differentiating through the solver.}\label{fig:implicit_bias}
\end{wrapfigure} This is intuitively illustrated in Figure \ref{fig:implicit_bias}: while moving in direction of the gradient $\nabla_{x(1)}\gL(x(1))$ generally moves away from the data distribution (in pink), differentiating w.r.t.~$x(0)$ projects this gradient onto high variance data directions and consequently staying close to the data distribution. To exemplify this phenomena we show in Figure \ref{fig:implicit_bias_optimization} optimization steps $x^{(0)}(1),x^{(2)}(1),x^{(4)}(1),\ldots$ of a loss $\gL(x)=\norm{H(x)-H(x_*)}^2$, where $H$ is a linear matrix that subsamples a (random) subset of the image's pixels consisting of $90\%$ of the total number of pixels, and $x_*$ is a target image (different from the initial $x^{(0)}(1)$). The sampling process here is using an ImageNet trained flow model with the class condition `bulbul'. As can be seen in this sequence of images, the intermediate steps of the optimization stay close to the distribution and pass through different sub-species of the bulbul bird. In the next section we provide a precise mathematical statement supporting this claim but for now let us provide some intuitive explanation. 

\subsection{Practical Implementation}\label{s:practical_implementation}
The practical implementation of Algorithm \ref{alg:main} requires three algorithmic choices. First, one needs to decide how to initialize $x_0$. In all experiments we either initialize $x_0$ as a sample from the source distribution, \ie, normal Gaussian, or we use a variance preserving blend of a normal Gaussian with the backward solution from $t=1$ to $t=0$ of the observed signal when possible. Second, we need to choose the solver used to parameterize $x(1)$. To this end we utilize the \texttt{torchdiffeq} package \cite{torchdiffeq}, providing a wide class of differentiable ODE solvers. Backpropagating through the solver can be expensive in memory and we therefore use gradient checkpointing to reduce memory consumption at the cost of runtime. In most of our experiments we use the midpoint method with 6 function evaluations. Lastly, we need to choose the optimizer for the gradient step. Since the optimization we perform is not stochastic we choose to use the LBFGS algorithm with line search in all experiments.  The runtime of the optimization depends on the problem but typically ranges from $5-15$ minutes per sample. For large text-2-image and text-2-audio models run times are higher and can reach $30-40$ minutes. 

\section{Theory}
In this section we provide the theoretical support to the implicit regularization claim made in the previous section. 
 
First, we revisit the family of Affine Gaussian Probability Paths (AGPP) taking noise to data that are used to supervise diffusion/flow models. When diffusion/flow models reach zero loss they reproduce these probability paths and we will therefore use them to analyze the implicit bias. Second, we use the method of adjoint dynamics to provide an explicit formula for the gradient $\nabla_{x_0}\gL(x(1))$ under the AGPP assumption, and consequently derive the asymptotic change (velocity vector) in $x(1)$. Lastly, we interpret this velocity vector of $x(1)$ to demonstrate why it is pointing in the direction of the data distribution $p_1(x)$. 

\color{black}
\paragraph{Affine Gaussian probability paths.} 
Diffusion and recent flow based models use Affine Gaussian Probability Path (AGPP) to supervise their training. In particular, denoting $p_0=\gN(0,\sigma_0^2 I)$ the Gaussian noise (source) distribution and $p_1$ data (target) distribution, an AGPP is defined by 
\begin{equation}
    p_t(x) = \int p_t(x|x_1) p_1(x_1) dx_1,
\end{equation}
where $p_t(x|x_1)=\gN(x|\alpha_t x_1, \sigma_t^2 I)$ is a Gaussian kernel and $\alpha_t,\sigma_t:[0,1]\too [0,1]$ are called the \emph{scheduler}, satisfying $\alpha_0=0$, $\sigma_1\approx 0$, and $\alpha_1=1=\sigma_0$, consequently guaranteeing that $p_t$ interpolates (exactly or approximately) the source and target distributions at times $t=0$ and $t=1$, respectively. The velocity field that generates this probability path and coincide with the velocity field trained by diffusion/flow models at zero loss is \cite{lipman2023flow,shaul2023kinetic} 
\begin{equation}\label{e:u_t}
    u_t(x) = \int \brac{a_t x + b_t x_1} p_t(x_1|x) dx_1
\end{equation}
where using Bayes' Theorem 
\begin{equation}
p_t(x_1|x)=\frac{p_t(x|x_1)p_1(x_1)}{p_t(x)}, 
\end{equation}
and 
\begin{equation}\label{e:a_t_b_t}
a_t= \frac{\dot{\sigma}_t}{\sigma_t}, \quad    b_t = \dot{\alpha}_t - \alpha_t\frac{\dot{\sigma}_t}{\sigma_t}.    
\end{equation}

One can also simplify the integral in \eqref{e:u_t} and write the marginal vector field in terms of the denoiser \citep{karras2022elucidating}, $\hat{x}_{1|t}(x)=\int x_1 p_t(x_1|x)dx_1$:
\begin{equation}\label{e:u_t_denoiser}
    u_t(x) = a_t x + b_t \hat{x}_{1|t}(x) 
\end{equation}

which for AGPP possesses a useful property we will use in Theorem \ref{thm:D_x0_x1}, stated in the following proposition (proof in Appendix \ref{app:proof_prop_1}):

\begin{proposition}\label{prop:denoiser}
    For AGPP, the gradient of the denoiser $\hat{x}_{1|t}(x)$ w.r.t $x$ is proportional to the variance of the random variable defined by $p_t(x_1|x)$, formally:
    \begin{equation}\label{e:diff_denoiser_var}
    D_x\hat{x}_{1|t}(x) = \frac{\alpha_t}{\sigma_t^2} \Var_{1|t}(x)
\end{equation}
where 
\begin{equation}\label{e:var_definition}
\Var_{1|t}(x) = \E_{p_t(x_1|x)}\brac{x_1-\hat{x}_{1|t}(x)}\brac{x_1-\hat{x}_{1|t}(x)}^T
\end{equation}
\end{proposition}

\paragraph{Differentiating through the solver.} When diffusion/flow models are optimized to a minimal loss they perfectly reproduce the AGPP velocity field, \ie, \eqref{e:u_t} \cite{lipman2023flow}. For this velocity field we begin with an analysis of the differential of a solution (sample) $D_{x_0} x(1)$ for the continuous time exact case and a discrete time approximation. 

\begin{theorem}\label{thm:D_x0_x1}
    For AGPP velocity field $u_t$ (see \eqref{e:u_t}) and $x(t)$ defined via \eqref{e:dynamics} the differential of $x(1)$ as a function of $x_0$ is 
\begin{equation}\label{e:D_x0_x1}
        D_{x_0} x(1) =  \sigma_1 \gT\exp\brac{\int_0^1 \gamma_t \Var_{1|t}(x(t)) dt},
\end{equation}

where $\gT \exp[\cdot] $ stands for a time-ordered exponential, $\gamma_t=\frac{1}{2}\frac{d}{dt}\mathrm{snr}(t)$ and we define $\mathrm{snr}(t)=\frac{\alpha_t^2}{\sigma_t^2}$. 
\end{theorem}
 The proof is given in Appendix \ref{app:proof_thm_1}. In the exact case where $\sigma_1= 0$ we also have $\int_0^1\gamma_t dt=\infty$, nevertheless we show in Appendix \ref{app:proof_thm_1} that $D_{x_0}x(1)$ is the time-ordered exponential of a bounded time-dependent matrix. While a closed form expression to this integral is unknown, we note that the matrix-vector product $D_{x_0}x(1) v$ corresponds to an infinite sum of powers of the matrices $\gamma_t \Var_{1|t}(x(t))$ applied to $v$. 

To gain better intuition and align our theory with practice, where discrete ODE solvers are used to obtain $x(1)$, we will now analyze the discrete time solver case. Let us consider an Euler ODE solver with $N$ uniform steps of size $h=\frac{1}{N}$, then the differential of $x(1)$ as a function of $x_0$ is:
\begin{equation}\label{e:discrete_dx1_dx0}
    D_{x_0} x(1) = \prod_{m=0}^{N-1} \Big({(1+ha_{mh}) I + h\gamma_{mh}\Var_{1|mh}(x_{mh})}\Big)
\end{equation}

note that the product is a time-ordered product, with $m$ decreasing from right to left (derivation in Appendix \ref{app:discrete_time}). The form of \eqref{e:discrete_dx1_dx0}, consisting of powers of $\Var_{1|t}(x)$, provides insights as to why D-Flow works even with a low number of solver steps. Intuitively, the vector-matrix multiplication $\Var_{1|t}(x)v$ projects $v$ on the major axes of the distribution of the data conditioned on $x$. As we will soon see, $D_{x_0}x(1)$ is key to understanding the implicit bias claim.

\paragraph{The dynamics of $x(1)$.} Consider an optimization step updating the optimized variable $x_0$ with a gradient step, \ie, $x^\tau_0 = x_0 - \tau \nabla_{x_0}\gL(x(1))$, where the gradient $\nabla_{x_0}\gL(x(1))$ can be now computed with the chain rule and \eqref{e:D_x0_x1},
\begin{equation}\label{e:nabla_L_x0}
    \nabla_{x_0} \gL(x(1)) =  D_{x_0}x(1)^T \nabla_{x(1)} \gL(x(1)),
\end{equation}
We can now ask: \emph{How is the sample $x(1)$ changing infinitesimally under this gradient step?} Denote by $\Phi:\Real^d\too\Real^d$ the map taking initial conditions $x_0$ to solutions of \eqref{e:dynamics} at $t=1$, \ie, $\Phi(x_0)=x(1)$. The \emph{variation} of $x(1)$ is
\begin{align*}
    \delta x(1) &= \frac{d}{d\tau}\Big\vert_{\tau=0}\Phi\parr{x_0 - \tau \nabla_{x_0}\gL(x(1))} \\
    &= -\brac{D_{x_0}x(1) D_{x_0}x(1)^T}\nabla_{x(1)}\gL(x(1)),
\end{align*}
where the first equality is the definition of variation and the second equality is using chain rule and \eqref{e:nabla_L_x0}. Indeed, the dynamics of $x(1)$ follow the projection of the gradient $\nabla_{x(1)} \gL(x(1))$ with the operator $D_{x_0}x(1)$ that iteratively  applies projection by the covariance matrix $\Var_{1|t}(x(t))$ at different times $t$ (equations \ref{e:D_x0_x1} and \ref{e:discrete_dx1_dx0}).

\color{black}
\begin{figure*}
  \begin{center}
  \begin{tabular}{@{\hspace{0pt}}c@{\hspace{3pt}}c@{\hspace{3pt}}c@{\hspace{3pt}}c@{\hspace{3pt}}c@{\hspace{3pt}}c@{\hspace{7pt}}c@{\hspace{3pt}}c@{\hspace{3pt}}c@{\hspace{3pt}}c@{\hspace{3pt}}c@{\hspace{3pt}}c@{\hspace{0pt}}}
       \includegraphics[width=0.16\columnwidth]{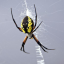} &  
       \includegraphics[width=0.16\columnwidth]{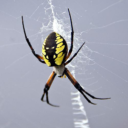} &
       \includegraphics[width=0.16\columnwidth]{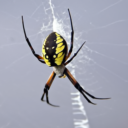} & 
       \includegraphics[width=0.16\columnwidth]{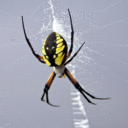} &
       \includegraphics[width=0.16\columnwidth]{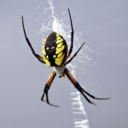} &
       \includegraphics[width=0.16\columnwidth]{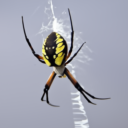} &
       \includegraphics[width=0.16\columnwidth]{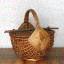} &  
       \includegraphics[width=0.16\columnwidth]{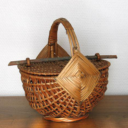} &
       \includegraphics[width=0.16\columnwidth]{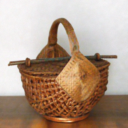} & 
       \includegraphics[width=0.16\columnwidth]{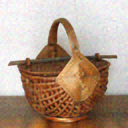} &
       \includegraphics[width=0.16\columnwidth]{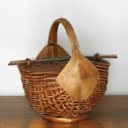} &
       \includegraphics[width=0.16\columnwidth]{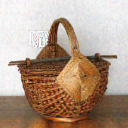}\\
       \includegraphics[width=0.16\columnwidth]{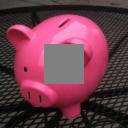} &  
       \includegraphics[width=0.16\columnwidth]{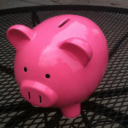} &
       \includegraphics[width=0.16\columnwidth]{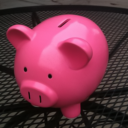} & 
       \includegraphics[width=0.16\columnwidth]{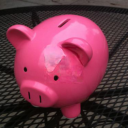} &
       \includegraphics[width=0.16\columnwidth]{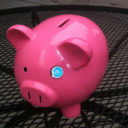} &
       \includegraphics[width=0.16\columnwidth]{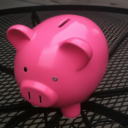} &
       \includegraphics[width=0.16\columnwidth]{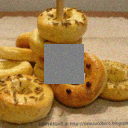} &  
       \includegraphics[width=0.16\columnwidth]{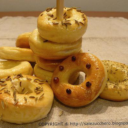} &
       \includegraphics[width=0.16\columnwidth]{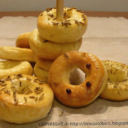} & 
       \includegraphics[width=0.16\columnwidth]{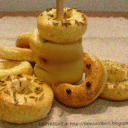} &
       \includegraphics[width=0.16\columnwidth]{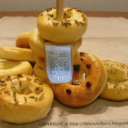} &
       \includegraphics[width=0.16\columnwidth]{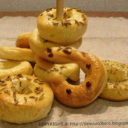}\\
       \includegraphics[width=0.16\columnwidth]{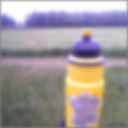} &  
       \includegraphics[width=0.16\columnwidth]{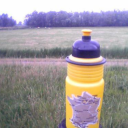} &
       \includegraphics[width=0.16\columnwidth]{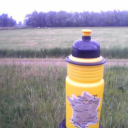} & 
       \includegraphics[width=0.16\columnwidth]{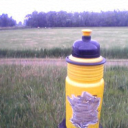} &
       \includegraphics[width=0.16\columnwidth]{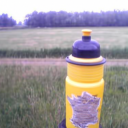} &
       \includegraphics[width=0.16\columnwidth]{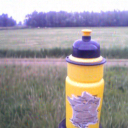} &
       \includegraphics[width=0.16\columnwidth]{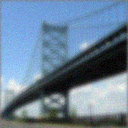} &  
       \includegraphics[width=0.16\columnwidth]{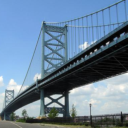} &
       \includegraphics[width=0.16\columnwidth]{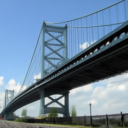} & 
       \includegraphics[width=0.16\columnwidth]{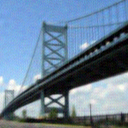} &
       \includegraphics[width=0.16\columnwidth]{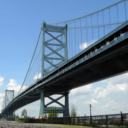} &
       \includegraphics[width=0.16\columnwidth]{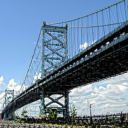}\\
       \scriptsize Distorted & \scriptsize Ground Truth & \scriptsize Ours & \scriptsize RED-Diff & \scriptsize  OT-ODE & \scriptsize $\Pi$GDM & \scriptsize Distorted & \scriptsize Ground Truth & \scriptsize Ours & \scriptsize RED-Diff & \scriptsize  OT-ODE & \scriptsize $\Pi$GDM  
  \end{tabular}    
  \end{center}
  \caption{Qualitative comparison for linear inverse problems on ImageNet-128. GT samples from ImageNet-128 validation.}\label{fig:linear_inv_imagenet}
\end{figure*}

\section{Related Work}

\textbf{Inverse Problems. } A new line of works alter the diffusion generation process for training-free solutions of inverse problems. Most works can be viewed as building guidance strategies to the generation process of diffusion models. \cite{kawar2022denoising} takes a variational approach deriving a solver for linear inverse problems. Similarly, \cite{chung2022improving,wang2022zeroshot} modify the generation process by enforcing consistency with the observations either via cost functions or projections \cite{choi2021ilvr,wang2022zeroshot,lugmayr2022repaint}. Other approaches guide the sampling process with derivatives through the diffusion model at each denoising step~\cite{ho2022video,chung2023diffusion,song2023pseudoinverseguided,pokle2023trainingfree}. A recent work by \cite{rout2023solving} extends the ideas for latent diffusion models by chained applications of encoder-decoder. Similar to our approach \cite{mardani2023variational} performs optimization of a reconstruction loss with score matching regularization.

\begin{table*}[t]
\caption{ Quantitative evaluation of linear inverse problems on face-blurred ImageNet-128.}
\resizebox{1.0\linewidth}{!}{%
\begin{tabular}{llcccccccccccccc}
\hline
                &  & \multicolumn{4}{c}{\textbf{Inpainting-Center}}                                                                                 &  & \multicolumn{4}{c}{\textbf{Super-Resolution X2}}                                                                                                                                                               & \multicolumn{1}{c}{} & \multicolumn{4}{c}{\textbf{Gaussian deblur}}                                                                                                                                                                   \\ \cline{3-6} \cline{8-11} \cline{13-16} 
Method          &  & FID $\downarrow$              & LPIPS $\downarrow$            & PSNR $\uparrow$                & SSIM $\uparrow$               &  & \multicolumn{1}{c}{FID $\downarrow$}              & \multicolumn{1}{c}{LPIPS $\downarrow$}            & \multicolumn{1}{c}{PSNR $\uparrow$}                & \multicolumn{1}{c}{SSIM $\uparrow$}               & \multicolumn{1}{c}{} & \multicolumn{1}{c}{FID $\downarrow$}              & \multicolumn{1}{c}{LPIPS $\downarrow$}            & \multicolumn{1}{c}{PSNR $\uparrow$}                & \multicolumn{1}{c}{SSIM $\uparrow$}               \\ \cline{1-1} \cline{3-6} \cline{8-11} \cline{13-16} 
$\sigma_y=0$    &  & \multicolumn{1}{l}{}          & \multicolumn{1}{l}{}          & \multicolumn{1}{l}{}           & \multicolumn{1}{l}{}          &  &                                                   &                                                   &                                                    &                                                   &                      &                                                   &                                                   &                                                    &                                                   \\
\;\;$\Pi$GDM \tiny{\cite{song2023pseudoinverseguided}}        &  & 5.73                        & 0.096                         & 36.89                          & 0.908                         &  & \multicolumn{1}{c}{6.01}                          & \multicolumn{1}{c}{0.104}                         & \multicolumn{1}{c}{\cellcolor[HTML]{E8E8E8}34.31}                          & \multicolumn{1}{c}{\cellcolor[HTML]{E8E8E8}0.911}                         &                      & \multicolumn{1}{c}{4.27}                          & \multicolumn{1}{c}{0.066}                         & \multicolumn{1}{c}{\cellcolor[HTML]{E8E8E8}37.61}  & \multicolumn{1}{c}{\cellcolor[HTML]{E8E8E8}0.961} \\
\;\;OT-ODE \tiny{\cite{pokle2023trainingfree}}          &  & 5.65                         & 0.094                         & 37.00                         & 0.893                         &  & \multicolumn{1}{c}{4.28}                          & \multicolumn{1}{c}{0.097}                         & \multicolumn{1}{c}{33.88}                          & \multicolumn{1}{c}{0.903}                         &                      & \multicolumn{1}{c}{\cellcolor[HTML]{E8E8E8}2.04}  & \multicolumn{1}{c}{\cellcolor[HTML]{E8E8E8}0.048} & \multicolumn{1}{c}{37.44}                          & \multicolumn{1}{c}{0.959} \\
\;\;RED-Diff \tiny{\cite{mardani2023variational}}         &  & \cellcolor[HTML]{E8E8E8}5.40 & \cellcolor[HTML]{90C4C7}0.068 & \cellcolor[HTML]{90C4C7}38.91 & \cellcolor[HTML]{90C4C7}0.928 &  & \multicolumn{1}{c}{\cellcolor[HTML]{E8E8E8}3.05}  & \multicolumn{1}{c}{\cellcolor[HTML]{E8E8E8}0.091} & \multicolumn{1}{c}{33.74}  & \multicolumn{1}{c}{0.900} &                      & \multicolumn{1}{c}{\cellcolor[HTML]{90C4C7}1.62}                          & \multicolumn{1}{c}{0.055}                         & \multicolumn{1}{c}{35.18}                          & \multicolumn{1}{c}{0.937}                         \\
\;\;Ours            &  & \cellcolor[HTML]{90C4C7}4.14 & \cellcolor[HTML]{E8E8E8}0.072 & \cellcolor[HTML]{E8E8E8}37.67 & \cellcolor[HTML]{E8E8E8}0.922 &  & \multicolumn{1}{c}{\cellcolor[HTML]{90C4C7}2.50} & \multicolumn{1}{c}{\cellcolor[HTML]{90C4C7}0.069} & \multicolumn{1}{c}{\cellcolor[HTML]{90C4C7}34.88} & \multicolumn{1}{c}{\cellcolor[HTML]{90C4C7}0.924} &                      & \multicolumn{1}{c}{2.37} & \multicolumn{1}{c}{\cellcolor[HTML]{90C4C7}0.035} & \multicolumn{1}{c}{\cellcolor[HTML]{90C4C7}39.47} & \multicolumn{1}{c}{\cellcolor[HTML]{90C4C7}0.976} \\ \cline{1-1} \cline{3-16} 
$\sigma_y=0.05$ &  & \multicolumn{1}{l}{}          & \multicolumn{1}{l}{}          & \multicolumn{1}{l}{}           & \multicolumn{1}{l}{}          &  &                                                   &                                                   &                                                    &                                                   &                      &                                                   &                                                   &                                                    &                                                   \\
\;\;$\Pi$GDM \tiny{\cite{song2023pseudoinverseguided}}        &  & 7.99                          & 0.122                         & 34.57                          & 0.867                         &  & \cellcolor[HTML]{E8E8E8}4.38                      & \cellcolor[HTML]{E8E8E8}0.148                     & 32.07                                              & 0.831                                             &                      & 30.30                                             & 0.328                                             & 29.96                                              & 0.606                                             \\
\;\;OT-ODE \tiny{\cite{pokle2023trainingfree}}           &  & \cellcolor[HTML]{E8E8E8}6.25  & \cellcolor[HTML]{E8E8E8}0.119 & \cellcolor[HTML]{90C4C7}35.01  & \cellcolor[HTML]{E8E8E8}0.882 &  & 4.61                     & 0.149                                             & \cellcolor[HTML]{90C4C7}32.59                     & \cellcolor[HTML]{90C4C7}0.862                     &                      & \cellcolor[HTML]{90C4C7}4.84 & \cellcolor[HTML]{E8E8E8}0.175 & \cellcolor[HTML]{E8E8E8}31.94 & \cellcolor[HTML]{90C4C7}0.821                                       \\
\;\;RED-Diff \tiny{\cite{mardani2023variational}}       &  & 14.63                         & 0.171                         & 32.42                          & 0.820                         &  & 10.54                                             & 0.182                                             & 31.82                                              & 0.852                                             &                      & 21.43                                             & 0.229                                             & 31.41                                              & 0.807                                             \\
\;\;Ours            &  & \cellcolor[HTML]{90C4C7}4.76 & \cellcolor[HTML]{90C4C7}0.102 & \cellcolor[HTML]{E8E8E8}34.609 & \cellcolor[HTML]{90C4C7}0.890 &  & \cellcolor[HTML]{90C4C7}4.26                                             & \cellcolor[HTML]{90C4C7}0.146                     & \cellcolor[HTML]{E8E8E8}32.35                    & \cellcolor[HTML]{E8E8E8}0.858                    &                      & \cellcolor[HTML]{E8E8E8}5.35 & \cellcolor[HTML]{90C4C7}0.167 & \cellcolor[HTML]{90C4C7}31.99 & \cellcolor[HTML]{E8E8E8}0.820                                               
\end{tabular}
}\vspace{10pt}
\label{tab:lin_inv}
\end{table*}

\textbf{Conditional sampling. } Conditional sampling from diffusion models can be achieved by training an additional noise-aware condition predictor model~\cite{song2020score} or by incorporating the condition into the training process \cite{dhariwal2021diffusion,ho2022classifier}. These approaches however require task specific training. Plug-and-play approaches, on the other hand, utilize a pre-trained unconditional generative model as a prior. \cite{graikos2023diffusion} perform constrained generation via optimization of a reconstruction term regularized by the diffusion loss. \cite{Liu_2023_CVPR} seeks for optimal control optimizing through the generation process to learn guiding controls. 
\color{black}
Our method formulates a similar optimization problem like earlier works on GANs \cite{bora2017compressed} and normalizing flows \cite{asim2020invertible, dombrowski2021diffeomorphic,chávez2022generative}. While \cite{asim2020invertible} provides an analysis of a simplified linear model, \cite{dombrowski2021diffeomorphic} analyzes the manifold preserving properties of diffeomorphic generative models. Our work provides a novel theoretical analysis of the gradient of differentiable functionals with respect to initial values of diffusion/flow generative processes with affine Gaussian paths. Our analysis unravels a fresh perspective on the implicit regularization implemented by differentiating through the generation process, even with a few number of steps (Appendix \ref{app:discrete_time}), that aligns with the denoising attributes of diffusion/flow models. We note that using gradients through the solver for the case of discrete diffusion models was first used by \cite{wallace2023endtoend} for classifier guidance and by \cite{samuel2023generating} to generate rare samples. 
\color{black}

\begin{table}[]
\caption{ Quantitative evaluation of free-form inpainting on MS-COCO with T2I latent model.\vspace{-8pt}}
\resizebox{1.0\linewidth}{!}{%
\begin{tabular}{llccccc}
\hline
                &  & \multicolumn{5}{c}{\textbf{Inpainting-Free-Form}} \\ \cline{3-7} 
Method          &  & FID $\downarrow$              & LPIPS $\downarrow$            & PSNR $\uparrow$                & SSIM $\uparrow$  & Clip score $\uparrow$                    \\ \cline{1-1} \cline{3-7} \vspace{1pt}

\;\;RED-Diff \tiny{\cite{mardani2023variational}}      &   & \cellcolor[HTML]{E8E8E8}23.31 & \cellcolor[HTML]{90C4C7}0.327 & \cellcolor[HTML]{90C4C7}33.28 & \cellcolor[HTML]{90C4C7}0.813 & \cellcolor[HTML]{E8E8E8}0.882   \\
\;\;Ours            &  & \cellcolor[HTML]{90C4C7}16.92 & \cellcolor[HTML]{90C4C7}0.327 & \cellcolor[HTML]{E8E8E8}32.34 & \cellcolor[HTML]{E8E8E8}0.759 & \cellcolor[HTML]{90C4C7}0.892                                                                               
\end{tabular}
}
\label{tab:latent_inv}
\end{table}

\section{Experiments}
We test D-Flow on the tasks: linear inverse problems on images, inverse problems with latent flow models and conditional molecule generation. For all the inverse problems experiments, where the observed signal provides structural information, we use a blend initialization to our algorithm speeding up convergence and often improving performance. Furthermore, in most experiments we find that there is no need in adding an explicit regularizing term in the optimization. The only cases where we found regularization helpful was in the noisy case for linear inverse problems and molecule generation. Additional details are in Appendix \ref{app:implementation}.

\subsection{Linear Inverse Problems on Images}
We validate our method on standard linear inverse problems with a known degradation model on images. The tasks we consider are center-crop inpainting, super-resolution and Gaussian deblurring both in the noiseless and noisy case. In all cases we stop the optimization at a task dependent target PSNR. For the noisy case we choose the target PSNR to be the PSNR corresponding to the known added noise.

\textbf{Tasks. } We follow the same settings as in \cite{pokle2023trainingfree}: (i) For center-crop inpainting, we use a $40\times 40$ centered mask; (ii) for super-resolution we use bicubic interpolation to downsample the images by $\times 2$; and lastly (iii) for Gaussian deblur we apply a Gaussian blur kernel of size $61\times61$ with intensity $1$. For each task we report results for the noiseless and noisy (Gaussian noise of $\sigma_y=0.05$, see \eqref{e:inverse_problem}) cases. Further implementation details can be found in the Appendix \ref{app:implementation_linear}.

\textbf{Metrics. } Following the evaluation protocol of prior works \cite{chung2022improving,kawar2022denoising} we report Fr\'echet Inception Distance (FID) \cite{heusel2018gans},  Learned Perceptual Image Patch Similarity (LPIPS) \cite{zhang2018unreasonable}, peak signal-to-noise ratio (PSNR), and structural similarity index (SSIM). 

\textbf{Datasets and baselines. } We use the face-blurred ImageNet-128 dataset and report our results on the $10k$ split of the face-blurred ImageNet dataset used by \cite{pokle2023trainingfree}. We compare our method to three recent state of the art methods: $\Pi$GDM \cite{song2023pseudoinverseguided}, OT-ODE \cite{pokle2023trainingfree} and RED-Diff \cite{mardani2023variational}. We use the implementation of \cite{pokle2023trainingfree} for all the baselines. All methods, including ours, are evaluated with the same Cond-OT flow-matching class conditioned model trained on the face-blurred ImageNet-128 unless the reults we produced were inferior to the ones reported in \cite{pokle2023trainingfree}. In that case, we use the reported numbers from \cite{pokle2023trainingfree}.

\begin{table*}[t]
\caption{Quantitative evaluation of music generation with latent flow models.\vspace{-8pt}}
\resizebox{1.0\linewidth}{!}{%
\begin{tabular}{lcccccccccccccccc}
                \toprule
                & \multicolumn{1}{c}{} & \multicolumn{2}{c}{\textbf{Inpainting (10\%)}} &  \multicolumn{1}{c}{} & \multicolumn{2}{c}{\textbf{Inpainting (20\%)}} & \multicolumn{1}{c}{} & \multicolumn{2}{c}{\textbf{Super-Resolution X2}} & \multicolumn{1}{c}{} & \multicolumn{2}{c}{\textbf{Super-Resolution X4}} & \multicolumn{1}{c}{} & \multicolumn{2}{c}{\textbf{Super-Resolution X8}}
                \\
                Method & & FAD $\downarrow$ & PSNR $\uparrow$ & & FAD $\downarrow$ &  PSNR $\uparrow$  & &FAD $\downarrow$ &  PSNR $\uparrow$  & &FAD $\downarrow$ &  PSNR $\uparrow$  & & FAD $\downarrow$ &  PSNR $\uparrow$ \\
                \cmidrule{1-1} \cmidrule{3-4} \cmidrule{6-7} \cmidrule{9-10} \cmidrule{12-13} \cmidrule{15-16}
                In-domain\\
                \;\;RED-Diff \tiny{\cite{mardani2023variational}} &&\cellcolor[HTML]{E8E8E8}0.75&\cellcolor[HTML]{90C4C7}31.19&&\cellcolor[HTML]{E8E8E8}0.78& \cellcolor[HTML]{90C4C7}29.99&&\cellcolor[HTML]{E8E8E8}0.93&\cellcolor[HTML]{E8E8E8}35.27&&\cellcolor[HTML]{E8E8E8}1.63&\cellcolor[HTML]{E8E8E8}33.51&&\cellcolor[HTML]{E8E8E8}1.73&\cellcolor[HTML]{E8E8E8}29.12\\
                \;\;Ours &&\cellcolor[HTML]{90C4C7}0.22&\cellcolor[HTML]{E8E8E8}31.02&&\cellcolor[HTML]{90C4C7}0.49 & \cellcolor[HTML]{E8E8E8}29.57&&\cellcolor[HTML]{90C4C7}0.22&\cellcolor[HTML]{90C4C7}44.51&&\cellcolor[HTML]{90C4C7}0.50&\cellcolor[HTML]{90C4C7}42.64&&\cellcolor[HTML]{90C4C7}1.01&\cellcolor[HTML]{90C4C7}36.50\\
                \midrule
                MusicCaps\\
                \;\;RED-Diff \tiny{\cite{mardani2023variational}} &&\cellcolor[HTML]{E8E8E8}3.59&\cellcolor[HTML]{90C4C7} 32.81&&\cellcolor[HTML]{E8E8E8}3.72& \cellcolor[HTML]{E8E8E8}30.39&&\cellcolor[HTML]{E8E8E8}3.07 &\cellcolor[HTML]{E8E8E8}37.13&&\cellcolor[HTML]{E8E8E8}3.51&\cellcolor[HTML]{E8E8E8}34.99&&\cellcolor[HTML]{E8E8E8}3.97&\cellcolor[HTML]{E8E8E8}30.49\\
                \;\;Ours &&\cellcolor[HTML]{90C4C7} 1.19&\cellcolor[HTML]{E8E8E8}31.78&&\cellcolor[HTML]{90C4C7} 1.31& \cellcolor[HTML]{90C4C7} 31.08&&\cellcolor[HTML]{90C4C7}1.25&\cellcolor[HTML]{90C4C7}38.93&&\cellcolor[HTML]{90C4C7}1.42&\cellcolor[HTML]{90C4C7}35.83&&\cellcolor[HTML]{90C4C7}2.09&\cellcolor[HTML]{90C4C7}32.20\\
\bottomrule 
\end{tabular}
}
\label{tab:audio_inv}
\end{table*}

\textbf{Results. } As shown in Table \ref{tab:lin_inv}, our method shows strong performance across all tasks, Figure \ref{fig:linear_inv_imagenet} shows samples for each type of distortion. For inpainting and super-resolution our method improves upon state of the art in most metrics. We believe that our method's ability to reach images with higher fidelity to the ground truth is attributed to the source point optimization, which, differently from guided sampling approaches such as \citep{song2023pseudoinverseguided,pokle2023trainingfree}, iteratively correct the sampling trajectory to better match the observed signal. We further note that compared to RED-Diff, which is also an optimization approach, our method does not struggle in the noisy case and achieves SOTA performance. We show more samples in Figures \ref{afig:linear_inv_imagenet},\ref{afig:linear_inv_imagenet_2}.

\subsection{Inverse Problems with Latent Flow Models}

\subsubsection{Image Inpainting}
We demonstrate the capability of our approach for non-linear inverse problems by applying it to the task of free-form inpainting using a latent T2I FM model. 

\textbf{Metrics.} To quantitatively assess our results we report standard metrics used in T2I generation: PSNR, FID \cite{heusel2018gans}, and Clip score \cite{ramesh2022hierarchical}.

\textbf{Datasets and baselines. } The T2I model we use was trained on a proprietary dataset of $330m$ image-text pairs. It was trained on the latent space of an autoencoder as in \cite{rombach2022highresolution}. The architecture is based on GLIDE \cite{nichol2022glide} and uses a T5 text encoder \cite{raffel2023exploring}. We evaluate on a subset of $1k$ samples from the validation set of the COCO dataset \cite{lin2015microsoft}. We compare our method to RED-Diff \cite{mardani2023variational} as it is also not limited to linear inverse problems like the other baselines we used in the previous section. We tested different hyper-parameters for RED-Diff and report results with the best. 

\textbf{Results.} Table \ref{tab:latent_inv} reports metrics for the baseline and our method. The metrics indicate that while RED-Diff better matches the unmasked areas, achieving superior performance for structural metrics (PSNR, SSIM) our method produces more semantically plausible image completion winning in perceptual metrics. We do observe that RED-Diff often produces artifacts for this task. Results are visualized in Figure \ref{afig:inp_latent_image}.

\subsubsection{Audio Inpainting and Super-Resolution}

We evaluate our method on the tasks of music inpainting and super-resolution, utilizing a latent flow-matching music generation model. For this, we used a trained Cond-OT flow-matching text conditioned model with a transformer architecture of 325m parameters that operates on top of EnCodec representation~\cite{defossez2022highfi}. The model's performance aligns with the current state-of-the-art scores in text-conditional music generation, achieving  a Fr\'echet Audio Distance (FAD) score of $3.13$~\citep{kilgour2018fr} on MusicCaps and FAD of $0.72$ on in-domain data. The model is trained to generate ten-seconds samples. In the following, we evaluate the performance of inpainting and super-resolution using our method and RED-Diff as baseline, we report FAD and PSNR metrics.

\textbf{Datasets and baselines.} 
For evaluation, we use the MusicCaps benchmark, which comprises of $5.5$K pairs of music and a textual description and an internal (in-domain) evaluation set of $202$ samples, similar to~\citep{copet2023simple,ziv2024masked}. Similar to prior work, we compute FAD metric using VGGish. We compare our method to RED-Diff~\citep{mardani2023variational}.

\begin{figure*}[t]
  \begin{center}
  \begin{tabular}{@{\hspace{0pt}}c@{\hspace{3pt}}c@{\hspace{0pt}}@{\hspace{3pt}}c@{\hspace{0pt}}@{\hspace{3pt}}c@{\hspace{0pt}}@{\hspace{3pt}}c@{\hspace{0pt}}@{\hspace{3pt}}c@{\hspace{0pt}}@{\hspace{3pt}}c@{\hspace{0pt}}@{\hspace{3pt}}c@{\hspace{0pt}}}
       \includegraphics[width=0.22\columnwidth]{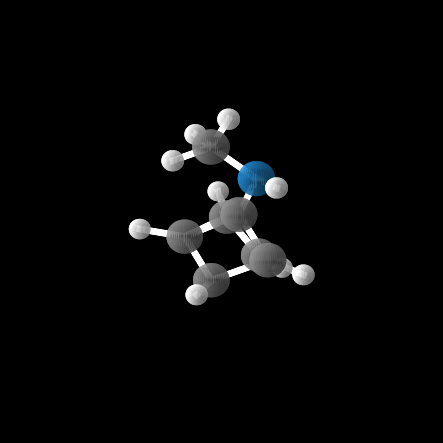} &  
       \includegraphics[width=0.22\columnwidth]{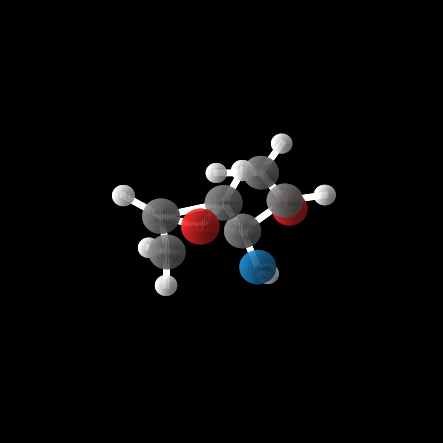} &  
       \includegraphics[width=0.22\columnwidth]{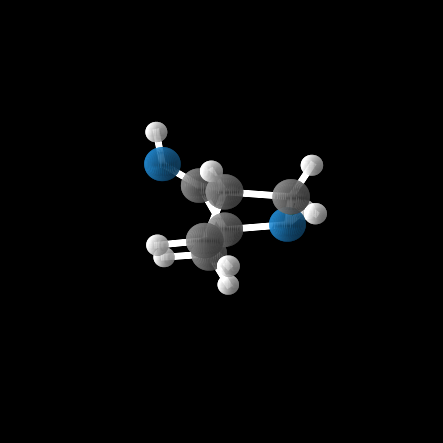} &
       \includegraphics[width=0.22\columnwidth]{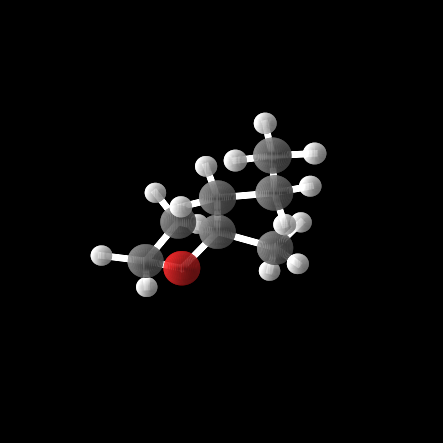} & 
       \includegraphics[width=0.22\columnwidth]{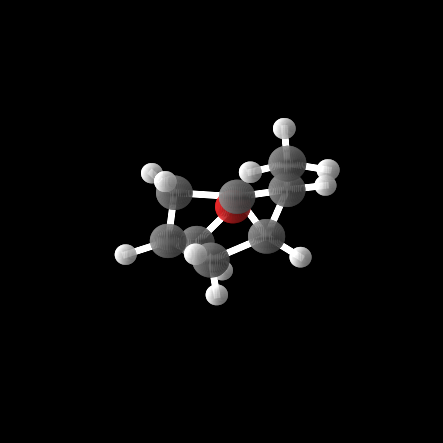} &
       \includegraphics[width=0.22\columnwidth]{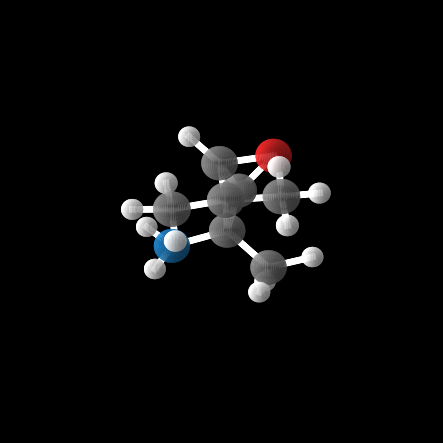} &
       \includegraphics[width=0.22\columnwidth]{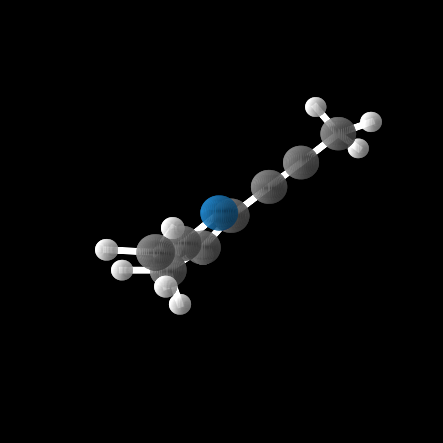} & 
       \includegraphics[width=0.22\columnwidth]{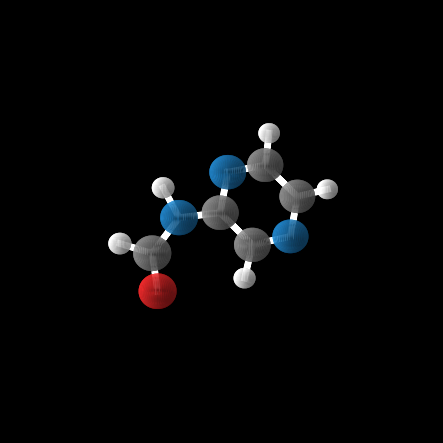}
       \\
       $52.22$ & $63.60$ &  $72.63$ & $79.36$ & $81.33$ &  $88.88$ & $95.13$ & $102.40$\vspace{-15pt}
  \end{tabular} 
  \end{center}
  \caption{Qualitative visualization of controlled generated molecules for various polarizability ($\alpha$) levels. \vspace{-10pt}}\label{fig:qm9}
\end{figure*}

\textbf{Results.} Table~\ref{tab:audio_inv} studies our method in inapinting and super resolution tasks. This experiment demosntrates the ability of our method to work in non-linear setup, where the flow model is trained over a neural representation and the cost function is evaluated on the post-decoded signal (neural representation after decoding). In the inpainting task, we center crop the signal by $10$\% and $20$\%, i.e., for a ten-seconds signal, we mask out two and four seconds respectively. In the super-resolution task we upscale a signal by factors of two, four, and eight,\ie , from $4$kHz, $8$kHz, $16$kHz to $32$kHz respectively. Overall, our method improves upon the baseline. Specifically, in all experiments, our method obtain the lowest FAD metric. In the inpainting task our method obtains a slightly lower PSNR from the baseline. Audio samples are attached in a supplementray material. Additional implementation details appear in Appendix~\ref{app:implementation_latent_aud}.

\subsection{Conditional Molecule Generation on QM9}

\color{black}

In this experiment we illustrate the application of our method for controllable molecule generation, which is of practical significance in the fields of material and drug design.  
The properties targeted for conditional generation ($c$ in \eqref{eq:cond_sample}) include polarizability $\alpha$, orbital energies $\varepsilon_{HOMO},\varepsilon_{LUMO}$ and their gap $\Delta\varepsilon$, Diople moment $\mu$, and heat capacity $C_v$. To assess the properties of the molecules generated, we used a property classifier ($\gF$ in \eqref{eq:cond_sample}) for each property. Those classifiers were trained following the methodology outlined in \cite{hoogeboom2022equivariant}. Further details are in Appendix \ref{app:implementation_latent_qm9}.

\begin{table}
    \caption{Quantitative evaluation of conditional molecule generation. Values reported in the table are MAE (over $10$K samples) for molecule property predictions (lower is better).}
    \resizebox{1.0\linewidth}{!}{%
    \centering
    \begin{tabular}{lcccccc}
    \toprule
        Property & $\alpha$ & $\Delta\varepsilon$ & $\varepsilon_{HOMO}$ & $\varepsilon_{LUMO}$ & $\mu$ & $C_v$ \\
        Units & Bohr$^2$  & meV & meV & meV & D  & $\frac{\text{cal}}{\text{mol}}$K \vspace{3pt}\\  \hline
         QM$9^*$ & 0.10 & 64 & 39 & 36 & 0.043 & 0.040 \\\hline
         EDM & 2.76 & 655 & 356 & 584 & 1.111 & 1.101\\
         E\scriptsize{QUI}\normalsize{FM}& 2.41 &  591  & \cellcolor[HTML]{E8E8E8} 337  &  530  & \cellcolor[HTML]{E8E8E8} 1.106  &  1.033 \\
         G\scriptsize{EO}\normalsize{LDM}& \cellcolor[HTML]{E8E8E8} 2.37   & \cellcolor[HTML]{E8E8E8} 587  &  340  & \cellcolor[HTML]{E8E8E8} 522  &  1.108  & \cellcolor[HTML]{E8E8E8} 1.025  \\\hline
         Ours & \cellcolor[HTML]{90C4C7} {1.39}  & \cellcolor[HTML]{90C4C7} {344}  & \cellcolor[HTML]{90C4C7} {182}  & \cellcolor[HTML]{90C4C7} {330}  & \cellcolor[HTML]{90C4C7} {0.300}  & \cellcolor[HTML]{90C4C7} {0.784}  \\
         \bottomrule  \vspace{-15pt}
    \end{tabular}
    }
    \label{tab:qm9}
\end{table}

\begin{table}
    \caption{Stability and validity evaluation of D-flow on conditional molecule generation ($10$K samples). \vspace{-8pt}}
    \resizebox{1.0\linewidth}{!}{%
    \centering
    \begin{tabular}{lcccccc}
    \toprule
        Property & $\alpha$ & $\Delta\varepsilon$ & $\varepsilon_{HOMO}$ & $\varepsilon_{LUMO}$ & $\mu$ & $C_v$ \\ \hline
        Molecule Stability (\%) & 56.2 &  59.4  &  60.2 &  59.4  &  60.7 &  57.9  \\
        Atom Stability (\%) & 93.6 &  93.9  &  94.1 &  93.8  &  94.2 &  93.6  \\
        Validity (\%) & 77.4 &  79.4  & 80.2 & 79.4 & 81.1 &  78.9 \\
        Validity \& Uniqueness (\%) & 77.4 &  79.4  & 80.2 & 79.4 & 81.1 &  78.9 \\  
         \bottomrule 
    \end{tabular}
    }
    \label{tab:qm9_stab_&_valid}
\end{table}

\textbf{Metrics.}
To assess conditional generation, we calculate the Mean Absolute Error (MAE) between the predicted property value of the generated molecule by the property classifier, \cite{satorras2022en}, and the target property value. According to the conditional training protocol from \cite{hoogeboom2022equivariant}, the property classifier is trained over half of the QM$9$ train set ($50$K) while the remaining half is used for training the conditional generative models. Additionally, we appraise the quality of the generated molecules by evaluating atom stability (the percentage of atoms with correct valency), molecule stability (the percentage of molecules where all atoms are stable), validity (as defined in RDKit \cite{Landrum2016RDKit2016_09_4}), and the uniqueness of the generated molecules.

\textbf{Dataset and baselines.} The generative models used for this experiment are trained using the QM$9$ dataset \cite{Ramakrishnan2014qm9}, a commonly used molecular dataset containing small molecules with up to $29$ atoms. 
The model we use as prior in these experiments is an unconditional equivariant Flow-Matching model with CondOT path \cite{lipman2023flow}, trained on the train set half used in \cite{hoogeboom2022equivariant} for conditional training. 
We compare our method to several state of the art \emph{conditional models}: conditional EDM, Equivariant Flow-Matching (E\scriptsize{QUI}\normalsize{FM}) \cite{song2023equivariant}, and Geometric Latent Diffusion Model (G\scriptsize{EO}\normalsize{LDM})\cite{xu2023geometric} an equivariant latent diffusion model. Additionally, we report the test MAE of each property classifier (denoted as QM$9^*$ in Table \ref{tab:qm9}), which serves as an empirical lower bound. It is important to note that for each specific property of conditional generation, the baseline methods utilized a distinct conditional model, each individually trained for generating that particular property while we used a single unconditional model. 

\textbf{Results.} Table \ref{tab:qm9} demonstrates that our approach significantly outperforms all other baseline methods in the quality of conditional molecule generation. This superior performance is attributed to our direct optimization of the conditional generation. 
Table \ref{tab:qm9_stab_&_valid} presents the stability and validity metrics for our method. In comparison with conditional EDM, which achieves an average molecular stability of $82.1\%$ across different properties, our method reveals a disparity in the stability of the generated molecules. This gap is a consequence of two factors. First, the trained Flow Matching unconditional model achieved inferior performance compared to EDM 
reaching molecular stability of $72.2\%$.  Second, the optimization with respect to the property predictor does not achieve the same quality of generation as regular sampling. We further verify that the gain in MAE that D-Flow presents is not due to the degradation in the percentage of stable molecules and report both MAE values for stable and non-stable molecules withing the $10k$ generated sample, in Table \ref{atab:qm9}. The MAE values for both stable and non-stable molecules are on par and improve by a large margin the existing baselines.
Figure \ref{fig:qm9} visualize the controlled generation for different polarizability $\alpha$ values; all molecules in the figure are valid and stable with a classifier error lower than $1$.
\color{black}

\section{Discussion, Limitations and Future Work}
We have presented a simple and general framework for controlled generation from pre-trained diffusion/flow models and demonstrated its efficacy on a wide range of problems from various domains and data types ranging from images, and audio to molecules. The main limitation of our approach is in its relatively long runtimes (see Section \ref{s:practical_implementation}, and Appendix \ref{app:implementation}) which stems from the need to back-propagate through multiple compositions of the velocity field (equivalently, the diffusion model). Our theoretical analysis and empirical evidence show however that computing gradients through the ODE solution have a desirable implicit bias, producing state of the art results on common conditional generation tasks. Consequently, an interesting future direction is to utilize the implicit bias but with potentially cheaper computational overhead, and draw connections to other biases used in other controlled generation paradigms. 

\section*{Acknowledgments}
OP is supported by a grant from Israel CHE Program for Data Science Research Centers and the Minerva Stiftung.

\section*{Impact Statement}
In this paper we introduce a general approach for controlled data generation from generative priors. This paper presents work whose goal is to advance the field of Machine Learning. There are many potential societal consequences of our work, none which we feel must be specifically highlighted here.

\bibliography{refs}
\bibliographystyle{icml2024}

\newpage
\appendix
\onecolumn

\section{Proofs and Theorems} \label{app:proofs}

\subsection{Proof of Proposition \ref{prop:denoiser}}\label{app:proof_prop_1}
\color{black}
We restate Proposition \ref{prop:denoiser} here:
\begin{proposition}
    For AGPP, the gradient of the denoiser $\hat{x}_{1|t}(x)$ w.r.t $x$ is proportional to the variance of the random variable defined by $p_t(x_1|x)$, formally:
    \begin{equation}\label{e:diff_denoiser_var}
    D_x\hat{x}_{1|t}(x) = \frac{\alpha_t}{\sigma_t^2} \Var_{1|t}(x)
\end{equation}
where 
\begin{equation}\label{e:var_definition}
\Var_{1|t}(x) = \E_{p_t(x_1|x)}\brac{x_1-\hat{x}_{1|t}(x)}\brac{x_1-\hat{x}_{1|t}(x)}^T
\end{equation}
\end{proposition}

\begin{proof}
    We recall a general affine Gaussian path is defined by 
\begin{align}
    p_t(x|x_1) &= \gN(x|\alpha_t x_1, \sigma_t^2 I), \qquad \quad  \ \ \ \text{conditional probability path} \\
    p_t(x) &= \int p_t(x|x_1) q(x_1) dx_1, \qquad  \text{marginal probability path} \label{ae:marginal_path_p}
\end{align}
where $(\alpha_t,\sigma_t)$ define the scheduler and $q$ is the dataset probability density. The velocity fields defining these paths are \cite{lipman2023flow}:
\begin{align}
    u_t(x|x_1) &= a_t x + b_t x_1, \qquad 
    a_t= \frac{\dot{\sigma}_t}{\sigma_t},  \ \ \  b_t = \dot{\alpha}_t - \alpha_t\frac{\dot{\sigma}_t}{\sigma_t} \qquad \qquad \qquad 
 \quad \text{conditional velocity field}  \label{ae:cond_path}\\
    u_t(x) &= \int u_t(x|x_1) p_t(x_1|x) dx_1, \qquad p_t(x_1|x)=\frac{p_t(x|x_1)q(x_1)}{p_t(x)} \qquad  \text{marginal velocity field}  \label{ae:marginal_path}
\end{align}

The differential of the denoiser is then:
\begin{equation}\label{ae:diff_denoiser}
    D_x\hat{x}_{1|t}(x)=D_x \int x_1 p_t(x_1|x) dx_1 = \int x_1 \nabla_x p_t(x_1|x) dx_1
\end{equation}

Since $p_t(x|x_1)$ is a Gaussian:
\begin{equation}
    \nabla_x p_t(x|x_1) = \frac{\alpha_t x_1 - x}{\sigma_t^2}p_t(x|x_1) 
\end{equation}

and plugging into \ref{ae:marginal_path_p}, we have:
\begin{equation}
    \nabla_x p_t(x) = \int \frac{\alpha_t x_1 - x}{\sigma_t^2}p_t(x|x_1)q(x_1)dx_1
\end{equation}

using \ref{ae:marginal_path}, we get:

\begin{equation}
    \nabla_x p_t(x_1|x) = p_t(x_1|x) \frac{\alpha_t}{\sigma_t^2} \parr{x_1-\hat{x}_{1|t}(x)}
\end{equation}

therefore, \ref{ae:diff_denoiser} takes the form:
\begin{align}\label{ae:diff_denoiser_var}
    D_x\hat{x}_{1|t}(x) &= \int \frac{\alpha_t}{\sigma_t^2} p_t(x_1|x) x_1(x_1-\hat{x}_{1|t}(x))^Tdx_1 = \\ &=\int \frac{\alpha_t}{\sigma_t^2}p_t(x_1|x) (x_1-\hat{x}_{1|t}(x))(x_1-\hat{x}_{1|t}(x))^Tdx_1 = \frac{\alpha_t}{\sigma_t^2} \Var_{1|t}(x) \nonumber
\end{align}
\end{proof}

\subsection{Proof of Theorem \ref{thm:D_x0_x1}}\label{app:proof_thm_1}

We restate Theorem \ref{thm:D_x0_x1} here:
\begin{theorem}
     For AGPP velocity field $u_t$ (see \eqref{e:u_t}) and $x(t)$ defined via \eqref{e:dynamics} the differential of $x(1)$ as a function of $x_0$ is 
\begin{equation}\label{ae:D_x0_x1}
        D_{x_0} x(1) =  \sigma_1 \gT\exp\brac{\int_0^1 \gamma_t \Var_{1|t}(x(t)) dt},
\end{equation}

where $\gT \exp[\cdot] $ stands for a time-ordered exponential, $\gamma_t=\frac{1}{2}\frac{d}{dt}\mathrm{snr}(t)$ and we define $\mathrm{snr}(t)=\frac{\alpha_t^2}{\sigma_t^2}$. 
\end{theorem}

\begin{proof}
    To compute the differential of $x(1)$ w.r.t the initial point $x_0$ we utilize adjoint dynamics. 
    
    Let us define the adjoint $p(t)=D_{x(t)}x(1)$. The dynamics of $p(t)$ are defined by the following ODE \cite{evans2005introduction}:
    \begin{align}\label{ae:adjoint_ode}
        \dot{p}(t) &= -D_x u_t(x(t))^T p(t) \\
        p(1) &= D_{x(1)}x(1) = I .\label{ae:adjoint_initial}
\end{align}

To compute $D_{x_0}x(1)$ we solve \ref{ae:adjoint_ode} from time $t=1$ back to time $t=0$.
Then,
\begin{equation}
    p(0) = D_{x_0}x(1).
\end{equation}

First, we will use the properties of AGPPs to further analyze the differential of the velocity field, $D_x u_t(x(t))$, that defines the adjoint dynamics in \eqref{ae:adjoint_ode}. 

We write the velocity field in terms of the denoiser by plugging \ref{ae:cond_path} into \ref{ae:marginal_path}:
\begin{equation}
    u_t(x) = a_t x + b_t \hat{x}_{1|t}(x) 
\end{equation}

and using \eqref{ae:diff_denoiser_var}, the differential of the velocity field is:
\begin{equation}\label{ea:D_x_u_t}
    D_x u_t(x) = a_t I + b_t D_x \hat{x}_{1|t}(x)= a_t I + b_t \frac{\alpha_t}{\sigma_t^2}\Var_{1|t}(x)
\end{equation}

The adjoint dynamics are then given by:
    \begin{align}\label{ae:adjoint_ode_}
        \dot{p}(t) &= A(t) p(t) \\
        p(1) &= D_{x(1)}x(1) = I .\label{ae:adjoint_initial_}
\end{align}

where 
\begin{align}\label{ae:At}
    A(t)=-\parr{a_tI +  \gamma_t\Var_{1|t}(x)}^T.
\end{align}

and we define:
\begin{align}
    \gamma_t=b_t\frac{\alpha_t}{\sigma_t^2} = \parr{\dot{\alpha}_t-\alpha_t\frac{\dot{\sigma}_t}{\sigma_t}}\frac{\alpha_t}{\sigma_t^2}=\frac{\dot{\alpha}_t\alpha_t\sigma_t^2-\alpha_t^2\dot{\sigma}_t\sigma_t}{\sigma_t^4}=\frac{1}{2}\frac{d}{dt}\parr{\frac{\alpha_t^2}{\sigma_t^2}} = \frac{1}{2}\frac{d}{dt}{\mathrm{snr}}(t)
\end{align}

The adjoint ODE \ref{ae:adjoint_ode_} is a non-autonomous linear ODE and together with the initial condition \eqref{ae:adjoint_initial_} its solution is given by \citep{Sakurai_Napolitano_2020}:
\begin{align}
    p(t) = \gT\exp\brac{-\int_t^1 A(s)ds}p(1)
\end{align}

known as the time-ordered exponential, defined as follows:
\begin{align}\label{ae:time_ordered_exp}
    \gT\exp\brac{\int_t^1 A(s)ds} = \sum_{n=1}^\infty \frac{(-1)^n}{n!} \int_t^1\dots \int_t^1 \gT\{A(s_1)A(s_2)\dots A(s_n)\}ds_1ds_2\dots ds_n
\end{align}

where $\gT\{A(s_1)A(s_2)\dots A(s_n)\}$ orders the product of matrices such that the value of $s_i$ decreases from right to left. Alternative equivalent solutions to this differential equation are the dyson series \citep{Sakurai_Napolitano_2020} and the Magnus expansion \citep{magnus}.

When the matrices $\{A(s)\}_{s\in[1,t]}$ commute, \ie, $[A(s),A(s')]=A(s)A(s')-A(s')A(s)=0$, the time-ordered exponential \ref{ae:time_ordered_exp} reduces to the first term in the sum for $n=1$ in \ref{ae:time_ordered_exp}. We can therefore separate the first term in \ref{ae:At}, since $I$ commutes with every matrix, and arrive at a simplified solution for $t=0$:

\begin{equation}
     D_{x_0}x(1) = \sigma_1 \gT\exp\brac{\int_0^1 \gamma_t\Var_{1|t}(x) dt},
 \end{equation}
 where $\exp\brac{\int_0^1 a_t dt}=\sigma_1$ under the assumption that $a_t$ is integrable, concluding the proof.



\end{proof}
 \color{black}
Next we show that the integral in \eqref{e:D_x0_x1} is defined also for $\sigma_1=0$. 
\begin{lemma}\label{lem:delta}
    For a Lipschitz function $f:\Real^d\too\Real$ we have that $\int \gN(x|y,\sigma^2I)f(x)dx = f(y)+\gO(\sigma)$.
\end{lemma}
\begin{proof}
    \begin{align*}
        \abs{\int \gN(x|y,\sigma^2I)f(x)dx - f(y)} &\leq 
        \int \gN(x|y,\sigma^2I))\abs{f(x)-f(y)}dx \\
        &= \int \gN(z|0,I)\abs{f(\sigma z +y) - f(y)} dz \\
        &\leq K \sigma\int \gN(z|0,I)\abs{ z} dz \\
        &= \gO(\sigma),
    \end{align*}
    where in the first equality we performed a change of variable $z=\frac{x-y}{\sigma}$, and in the second inequality we used the fact that $f$ is Lipschitz with constant $K>0$. 
\end{proof}

Using this Lemma we prove (under the assumption that $p_1(x)$ and its derivatives is Lipschitz):
\begin{proposition}\label{prop:denoiser_asym}
    The denoiser asymptotics at $t\too 1$ is 
    \begin{equation}
        \hat{x}_{1|t}(x) = \frac{x}{\alpha_t} + \gO(\sigma_t)
    \end{equation}
\end{proposition}
\begin{proof}
First we note that we assume $\sigma_t\too 0$ and $\alpha_t\too 1$ as $t\too 1$, 
\begin{equation}\label{ea:gaussian_change}
    \gN(x|\alpha_t x_1, \sigma_t^2I) = c_t\gN\parr{x_1\bigg|\frac{x}{\alpha_t},\parr{\frac{\sigma_t}{\alpha_t}}^2I},
\end{equation}
where $c_t$ is some normalization constant such that $c_1=1$. Now,
    \begin{align}
        p_t(x) &= \int \gN(x|\alpha_t x_1, \sigma_t^2I) p_1(x_1)dx_1 \\
        &= c_t\int \gN\parr{x_1\bigg|\frac{x}{\alpha_t},\parr{\frac{\sigma_t}{\alpha_t}}^2I} p_1(x_1)dx_1  \\ 
        &= c_t p_1\parr{\frac{x}{\alpha_t}} + \gO(\sigma_t),
    \end{align}
where in second equality we used \eqref{ea:gaussian_change} and the last equality Lemma \ref{lem:delta}.
\begin{align}
    \hat{x}_{1|t}(x) &= \int x_1 p_t(x_1|x) dx_1 \\
    &= \int x_1 \frac{\gN(x|\alpha_tx_1,\sigma_t^2I)p_1(x_1)}{p_t(x)} dx_1 \\
    &= \int x_1 \frac{c_t\gN\parr{x_1\Big|\frac{x}{\alpha_t},\parr{\frac{\sigma_t}{\alpha_t}}^2I}p_1(x_1)}{p_t(x)} dx_1 \\
    &= \frac{c_t\frac{x}{\alpha_t}p_1\parr{\frac{x}{\alpha_t}} + \gO(\sigma_t)} {c_t p_1\parr{\frac{x}{\alpha_t}} + \gO(\sigma_t)}\\
    &= \frac{x}{\alpha_t} + \gO(\sigma_t),
\end{align}
where in the second equality we used the definition of $p_t(x_1|x)$, in the third equality we used \eqref{ea:gaussian_change}, and in the fourth equality we used Lemma \ref{lem:delta}.\end{proof}

Now we can show that $D_x u_t(x(t))$ is bounded as $t\too 1$
\begin{align}
    D_x u_t(x(t)) &= a_t I + b_t D_x \hat{x}_{1|t}(x) \\
    &= a_tI + b_t\parr{\frac{1}{a_t}I + \gO(\sigma_t)} \\
    &= \frac{\dot{\alpha}_t}{\alpha_t}I + \gO(1),
\end{align}
where in the first equality we used \eqref{ea:D_x_u_t}, in the second Proposition \ref{prop:denoiser_asym} (and the fact that the derivatives of $p_1$ are Lipschitz for the derivation of the asymptotic rule), and in the last equality \eqref{ae:cond_path}. Furthermore $D_x u_t(x(t))$ is bounded as $t\too 0$ as both $a_0,b_0$ are well defined. This means that  $D_x u_t(x(t))$ is integrable over $[0,1]$.

\color{black}
\subsection{Discrete Time Analysis}\label{app:discrete_time}
Theorem \ref{thm:D_x0_x1} analyzes the continuous time case, providing intuition about the behavior of the dynamics of $x(1)$ when changing the initial condition $x_0$. The final expression \ref{e:D_x0_x1}, however, involves a time-ordered exponential which may be hard to interpret. Furthermore, our experiments show that even with a small number of discrete steps, differentiating $x(1)$ with respect to $x_0$ yields meaningful gradients, performing well in practice (see \ref{app:implementation}).

Let us consider Euler solver with $N$ uniform steps of size $h=\frac{1}{N}$, with initial point $x_0$. An intermediate point at time $mh$, $x_{mh}$ is given by:

\begin{equation}
    x_{(m+1)h} = x_{mh} + h u_{mh}(x_{mh} )
\end{equation}

We are interested at the derivative of $x_{Nh}=x(1)$ w.r.t $x_0$.

By the chain rule and \eqref{ea:D_x_u_t}, one can write:
\begin{equation}\label{ea:discrete_dx1_dx0}
    D_{x_0} x_1 = \prod_{m=0}^{N-1}  D_{x_{mh}} x_{(m+1)h} = \prod_{m=0}^{N-1} \Big({(1+ha_{mh}) I + h\gamma_{mh}\Var_{1|mh}(x_{mh})}\Big)
\end{equation}
note that this is also a time-ordered product, with $m$ decreasing from left to right.

For the CondOT probability path, where $\alpha_t=t, \sigma_t=1-t$, \ref{ea:discrete_dx1_dx0} takes the form:
\begin{equation}
     D_{x_0} x_1=\prod_{m=0}^{N-1} \Big( \frac{1-(m+1)h}{1-mh}I + \frac{mh^2}{(1-mh)^3}\Var_{1|mh}(x_{mh})\Big)
\end{equation}

\color{black}
\subsection{On Flow-Matching, Denoisers and Noise Prediction} 
Consider a general affine conditional probability path defined by the following transport map:
\begin{align*}
    x_t = \sigma_t x_0 + \alpha_t x_1
\end{align*}
where $x_0\sim p_0$ and $x_1\sim p_1$. 

For different choices of $\sigma_t,\alpha_t$ we can parametrize known diffusion and flow-matching paths. 
The corresponding conditional vector field on $x_1$ is:
\begin{align*}
    u_t(x|x_1) = \frac{\dot{\sigma}_t}{\sigma_t}(x-\alpha_t x_1) + \dot{\alpha}_t x_1 = \frac{\dot{\sigma}_t}{\sigma_t}x - \parr{\frac{\dot{\sigma}_t\alpha_t}{\sigma_t} - \dot{\alpha}_t}x_1
\end{align*}
 and the conditional vector field on $x_0$ is:

 \begin{align*}
    u_t(x|x_0) = \dot{\sigma}_tx_0+\frac{\dot{\alpha}_t}{\alpha_t}(x-\sigma_tx_0)=\frac{\dot{\alpha}_t}{\alpha_t}x - \parr{\frac{\dot{\alpha}_t\sigma_t}{\alpha_t}-\dot{\sigma}_t}x_0
\end{align*}

where $\dot{f}=\frac{d}{dt}f$.


Consider the marginal velocity field:
\begin{align*}
    u_t(x) = \int u_t(x|x_1) p_t(x_1|x) dx_1 = \int u_t(x|x_0) p_t(x_0|x) dx_0
\end{align*}

One can express it in terms of the optimal \emph{denoiser} function, $\hat{x}_{1|t}(x)$:

\begin{align}\label{ae:denoiser_to_ut}
    u_t(x) &= \frac{\dot{\sigma}_t}{\sigma_t}\int x p_t(x_1|x)dx_1 - \parr{\frac{\dot{\sigma}_t\alpha_t}{\sigma_t} - \dot{\alpha}_t}\int x_1 p_t(x_1|x)dx_1  =\frac{\dot{\sigma}_t}{\sigma_t}x - \parr{\frac{\dot{\sigma}_t\alpha_t}{\sigma_t} - \dot{\alpha}_t}\hat{x}_{1|t}(x)
\end{align}

For Cond-OT:
\begin{align}
    u_t(x) =\frac{\hat{x}_{1|t}(x)-x}{1-t}
\end{align}

Or, in terms of the optimal noise predictor, $\eps_t(x)$, like in DDPM:
\begin{align*}
    u_t(x) = \frac{\dot{\alpha}_t}{\alpha_t}\int x p_t(x_0|x)dx_0 -\parr{\frac{\dot{\alpha}_t\sigma_t}{\alpha_t}-\dot{\sigma}_t} \int x_0 p_t(x_0|x)dx_0 = \frac{\dot{\alpha}_t}{\alpha_t}x -\parr{\frac{\dot{\alpha}_t\sigma_t}{\alpha_t}-\dot{\sigma}_t} \eps_t(x)
\end{align*}
and for Cond-OT:
\begin{equation}\label{ae:eps_to_vf}
    u_t(x) = \frac{x-\eps_t(x) }{t}
\end{equation}

\newpage

\section{Implementation details} \label{app:implementation}
\subsection{Linear Inverse Problems on Images}\label{app:implementation_linear}

\textbf{Optimization details. } For all experiments in this section we used the LBFGS optimizer with 20 inner iterations for each optimization step with line search. Stopping criterion was set by a target PSNR value, varying for different tasks. The solver used was midpoint with 6 function evaluations. The losses, regularizations, initializations and stopping criterions of our algorithm for the linear inverse problems are listed in Table \ref{atab:lin_inv_spec}. In the Table $\chi^d$ regularization corresponds to equation \ref{e:chid_reg} and $\lambda$ denotes the coefficients used. 

\begin{table}[H] 
\center
\caption{Algorithmic choices for the ImageNet-128 linear inverse problems tasks.}
\resizebox{1.0\linewidth}{!}{%
\begin{tabular}{ccccccccccc}
\hline
               &  & \multicolumn{2}{c}{\textbf{Inpainting-Center}} &  & \multicolumn{2}{c}{\textbf{Super-Resolution X2}} &  & \multicolumn{2}{c}{\textbf{Gaussian Deblur}} &  \\ \cline{3-4} \cline{6-7} \cline{9-10}
               &  & $\sigma_y=0$         & $\sigma_y=0.05$         &  & $\sigma_y=0$          & $\sigma_y=0.05$          &  & $\sigma_y=0$        & $\sigma_y=0.05$        &  \\ \cline{1-1} \cline{3-4} \cline{6-7} \cline{9-11} 
Loss           &  & \multicolumn{2}{c}{$-\mathrm{PSNR}(Hx,y)$}                                               &  &      \multicolumn{2}{c}{$-\mathrm{PSNR}(Hx,y)$}                                            &  &  $-\mathrm{PSNR}(H^\dagger Hx,H^\dagger y)$                   &   $-\mathrm{PSNR}(Hx,y)$                     &  \\
Regularization &  &    None                  &   $\chi^d$, with $\lambda=0.01$                       &  &   None                    &   $\chi^d$, with $\lambda=0.01$                       &  &   None                  &   $\chi^d$, with $\lambda=0.01$                      &  \\
Initialization &  &           \multicolumn{2}{c}{$0.1$ blend}                                   &  &   \multicolumn{2}{c}{$0.1$ blend}                                               &  &                    \multicolumn{2}{c}{$0.1$ blend}                         & \\
Target PSNR &  &           45 & 32                                   &  &   55 & 32                                              &  &                  55 & 32                        & 
\end{tabular}
}
\label{atab:lin_inv_spec}
\end{table}

\textbf{Runtimes. } For noiseless tasks: inpainting center crop took on avarage $10$ minutes per image, super resolution took $12.5$ minutes per image and Gaussian deblurring took $15.5$ minutes per image.
 For the noisy tasks: inpainting center crop took on avarage $4$ minutes per image, super resolution took $2.5$ minute per image and Gaussian deblurring took $3.5$ minutes per image. Experiments ran on 32GB NVIDIA V100 GPU.

 Metrics are computed using the open source TorchMetrics library \cite{Detlefsen2022}.

\textbf{RED-Diff baseline. } To use the RED-Diff baseline with a FM cond-OT trained model we transform the velocity field to epsilon prediction according to \ref{ae:eps_to_vf}. We searched for working parameters and reported results that outperformed the results that were produced by \cite{pokle2023trainingfree} with an epsilon prediction model, otherwise we kept the number from \cite{pokle2023trainingfree}.
 
\subsection{Inpainting with Latent Flow Models}
\subsubsection{Image inpainting} \label{app:implementation_latent_img}
\textbf{Optimization details.} In this experiment we used the LBFGS optimizer with 20 inner iterations for each optimization step with line search. Stopping criterion was set by a runtime limit of $30$ minutes, but optimization usually convergences before. The solver used was midpoint with 6 function evaluations and the loss was negative $\text{PSNR}$ without regularization. We initialized the algorithm with a backward blend with $\alpha=0.25$. To facilitate the backpropagation through a large T2I model we use gradient checkpointing. 

The validation set of the COCO dataset, used for evaluation, was downloaded from \hyperlink{http://images.cocodataset.org/zips/val2017.zip}{http://images.cocodataset.org/zips/val2017.zip}.

\textbf{RED-Diff baseline. } To adapt RED-Diff to a latent space diffusion model, let us recall the loss used in RED-Diff:
\begin{equation}
    \ell(\mu) = \norm{y-f(\mu)}^2 + \lambda_t (\texttt{sg}\brac{\eps(x(t),t)-\eps})^T\mu
\end{equation}
where $f$ can be any differentiable function. In latent diffusion/flow model for inverse problems we can model $f$ as $f=H(\texttt{decode}(\mu))$, where $\texttt{decode}$ applies the decoder of the autoencoder used in the latent diffusion/flow model and $H$ is the corruption operator. We use $\text{lr}=0.25, \lambda=0.25$.

\subsubsection{Audio inpainting}\label{app:implementation_latent_aud}

\textbf{Optimization details.} We follow the same setup described in~\ref{app:implementation_latent_img}. Differently, we use $10$ inner iterations and stop after $100$ global iterations. We initialize the algorithm with a backward blend with $\alpha=0.1$.

\textbf{RED-Diff baseline. } We follow the same adaptation described above in \ref{app:implementation_latent_img}. We use $\text{lr}=0.05, \lambda=0.5$.

\subsection{Conditional Molecule Generation on QM9}\label{app:implementation_latent_qm9}

\textbf{Optimization details.} In this section, we describe how Algorithm \ref{alg:main} was practically applied in the QM$9$ experiment. We initialized $x_0\in\Real^{n\times 9}$ for the experiment, where $n$ represents the molecule's atom count and $9$ the number of attributes per atom, using a standard Gaussian distribution. To enhance optimization process stability, we ensured $x_0$ had a feature-wise mean of zero and a standard deviation of one by normalizing it after every optimization step. We employed the midpoint method for the ode solver, with a total of $100$ function evaluations, \ie step size of $1/50$. The optimization technique utilized was LBFGS with line search, configured with $5$ optimization steps and a limit of $5$ inner iterations for each step. The learning rate was set to $1$. On average, generating a single molecule took approximately $2.5$ minutes using a single NVIDIA Quadro RTX8000 GPU.

\begin{table}[h]
\centering
\caption{Comparison of generated molecules quality using different solvers and D-Flow. }
\renewcommand{\tabcolsep}{1.8pt}
\resizebox{1.0\linewidth}{!}{
\begin{tabular}{lccccc}
\toprule
\begin{tabular}[c]{@{}l@{}}Sample Method\\  \end{tabular} & \begin{tabular}[c]{@{}c@{}}NFE\\ (\#) \end{tabular} & \begin{tabular}[c]{@{}c@{}}Molecule Stability\\  (\%)\end{tabular} & \begin{tabular}[c]{@{}c@{}}Atom Stability\\  (\%)\end{tabular} & \begin{tabular}[c]{@{}c@{}}Validity\\  (\%)\end{tabular} & \begin{tabular}[c]{@{}c@{}}Validity \& Uniqueness\\  (\%)\end{tabular} \\ \hline
Dopri5 Adaptive Solver & - & 72.03 & 96.14 & 85.00 & 83.84 \\
Midpoint (50 steps) & 100 & 72.10 & 96.18 & 85.56 & 84.39 \\ \hline
Midpoint  (50 steps) + optimization & 100 & 58.97 & 93.87 & 79.38 & 79.38\\
\bottomrule
\end{tabular} 
}
\label{tab:sampling_comp}
\end{table}
\color{black}
In Table \ref{atab:qm9} below we report MAE values over the split to stable and non-stable molecules within the $10k$ generated samples. Other baselines and the result denoted as 'Ours' in the table report MAE on the entire $10k$ set of molecules without distinguishing between stable and non-stable ones. Since our method produces lower stability percentage, we also report the MAE on the stable and non-stable splits. It can be seen that our improved MAE performance is not due to producing non-stable molecules with lower MAE, but performance is also SOTA on generated stable molecules.
\color{black}
\begin{table}[h]
    \caption{Quantitative evaluation of conditional molecule generation. Values reported in the table are MAE (over $10$K samples) for molecule property predictions (lower is better).}
     \centering
    \resizebox{0.5\linewidth}{!}{%
    \centering
    \begin{tabular}{lcccccc}
    \toprule
        Property & $\alpha$ & $\Delta\varepsilon$ & $\varepsilon_{HOMO}$ & $\varepsilon_{LUMO}$ & $\mu$ & $C_v$ \\
        Units & Bohr$^2$  & meV & meV & meV & D  & $\frac{\text{cal}}{\text{mol}}$K \vspace{3pt}\\  \hline
         QM$9^*$ & 0.10 & 64 & 39 & 36 & 0.043 & 0.040 \\\hline
         EDM & 2.76 & 655 & 356 & 584 & 1.111 & 1.101\\
         E\scriptsize{QUI}\normalsize{FM}& 2.41 &  591  &  337  &  530  &  1.106  &  1.033 \\
         G\scriptsize{EO}\normalsize{LDM}&  2.37   &  587  &  340  &  522  &  1.108  &  1.025  \\\hline
         Ours &  {1.39}  &  {344}  &  {182}  &  {330}  &  {0.300}  & {0.784}  \\
          Ours-stable & {1.40}  &  {347}  & {180}  & {337}  &  {0.287}  &  {0.835}  \\
           Ours-non stable & {1.38}  & {340}  & {186}  & {318}  & {0.321}  &  {0.714}  \\
         \bottomrule  \vspace{-15pt}
    \end{tabular}
    }
    \label{atab:qm9}
\end{table}

\textbf{QM9.}
The QM$9$ dataset \cite{Ramakrishnan2014qm9}, a widely recognized collection, encompasses molecular characteristics and atomic positions for $130$K small molecules, each containing no more than $9$ heavy atoms (up to $29$ atoms when including hydrogens). The train/validation/test partitions used are according to \cite{anderson2019cormorant} and consists of $100$K/$18$K/$13$ samples per partition. We provide additional details regarding the properties used in the experiment:
 \begin{itemize}
     \item $\alpha$ Polarizabilty - Tendency of a molecule to acquire an electric dipole moment when subjected to anexternal electric field. 
     \item $\varepsilon_{HOMO}$ - Highest occupied molecular energy.
     \item  $\varepsilon_{LUMO}$ - Lowest unoccupied molecular energy.
     \item  $\Delta\varepsilon$ - The difference between HOMO and LUMO.
     \item $\mu$ - Dipole moment.
     \item $C_v$ - Heat capacity at $298.15$K.
 \end{itemize}

\color{black}
\section{Additional Experiments and Ablations}

\subsection{Linear Inverse Problems on Images: Denoising}
In this experiment we consider the task of denoising. The corrupted signal, $y$, is given by $y= x + \eps,\; \eps\sim\gN(0,\sigma_yI)$, with $\sigma_y=0.05$. Hyperparameters are the same as in Table \ref{atab:lin_inv_spec} for the noisy experiments. Reported metrics are in Table \ref{tab:denoising}.

\begin{table*}[h]
\centering
\caption{ Quantitative evaluation of denoising inverse problem on face-blurred ImageNet-128.}
\resizebox{0.5\linewidth}{!}{%
\begin{tabular}{llcccc}
\hline
                &  & \multicolumn{4}{c}{\textbf{Denoising}}                                                                                                                                                                                                                                                  \\ \cline{3-6} 
Method          &  & FID $\downarrow$              & LPIPS $\downarrow$            & PSNR $\uparrow$                & SSIM $\uparrow$                              \\ \cline{1-1} \cline{3-6} 
 \cline{1-1} \cline{3-6} 
$\sigma_y=0.05$ &  & \multicolumn{1}{l}{}          & \multicolumn{1}{l}{}          & \multicolumn{1}{l}{}           & \multicolumn{1}{l}{}                                                    \\
\;\;$\Pi$GDM \tiny{\cite{song2023pseudoinverseguided}}        &  & 9.60                          & 0.107                        & 35.11                          & 0.903                          \\
\;\;OT-ODE \tiny{\cite{pokle2023trainingfree}}           &  & \cellcolor[HTML]{E8E8E8}3.14  & \cellcolor[HTML]{E8E8E8}0.062 & \cellcolor[HTML]{90C4C7}37.34  & \cellcolor[HTML]{90C4C7}0.964 \\
\;\;RED-Diff \tiny{\cite{mardani2023variational}}       &  & 9.19                         & 0.105                         & 32.52                         & 0.895                              \\
\;\;Ours            &  & \cellcolor[HTML]{90C4C7}2.83 & \cellcolor[HTML]{90C4C7}0.060 & \cellcolor[HTML]{E8E8E8}36.05 & \cellcolor[HTML]{E8E8E8}0.952                                             
\end{tabular}
}
\label{tab:denoising}
\end{table*}


\subsection{Ablation: Regularization Coefficient}

As reported in Table \ref{atab:lin_inv_spec}, on the tasks of linear inverse problems on images, we used the source point $\chi^d$ regularization, \eqref{e:chid_reg}, for the noisy case. In the plot below, \ref{afig:reg_coeff}, we report the evaluation metrics (FID, LPIPS, PSNR, SSIM) for varying regularization coefficient values, $\lambda$, on the task of noisy super-resolution. All other hyperparameters are as reported in Table \ref{atab:lin_inv_spec}. 

\begin{figure}[H]
  \begin{center}
  \begin{tabular}{cc}
       \includegraphics[width=0.35\columnwidth]{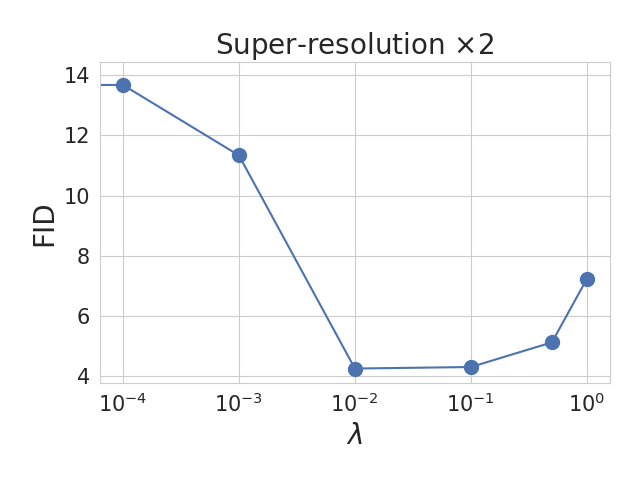} &  
       \includegraphics[width=0.35\columnwidth]{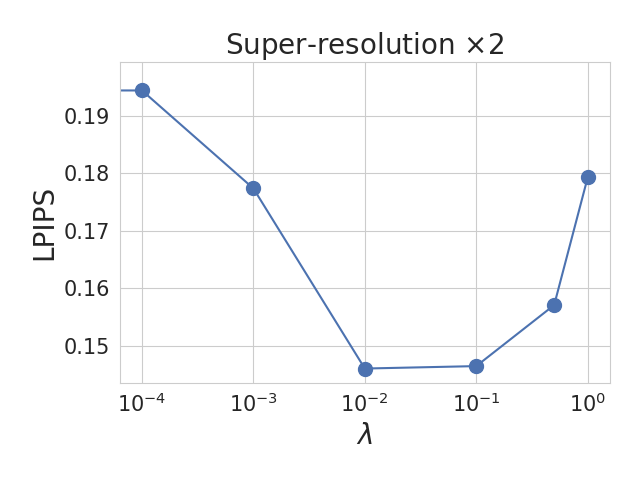} \\
       \includegraphics[width=0.35\columnwidth]{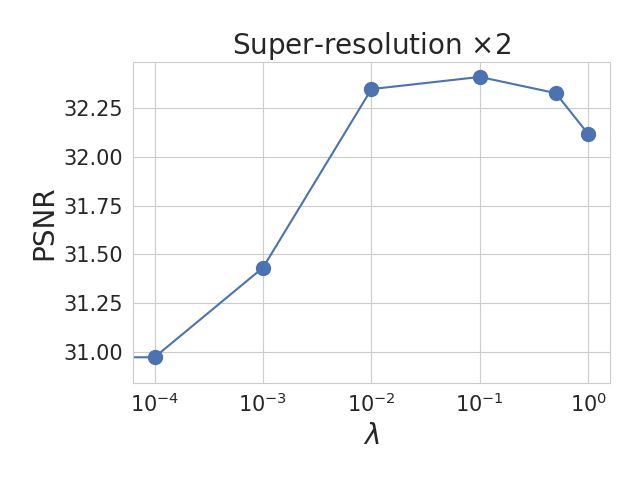} & 
       \includegraphics[width=0.35\columnwidth]{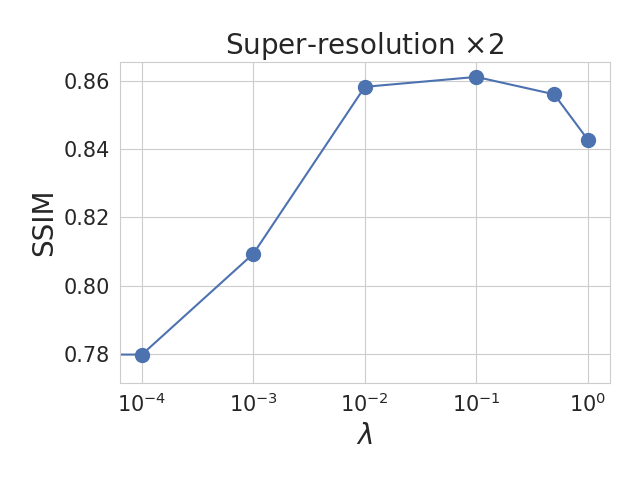}  \vspace{-10pt}
  \end{tabular}    
  \end{center}
  \caption{Evaluation metrics vs. regularization coefficient $\lambda$ of $\chi^d$ regularization over $x_0$ for noisy super-resolution on ace-blurred ImageNet-128.}\label{afig:reg_coeff}
\end{figure}

\color{black}


\newpage
\section{Additional Qualitative Results}
\begin{figure}[H]
  \begin{center}
  \begin{tabular}{@{\hspace{0pt}}c@{\hspace{3pt}}c@{\hspace{3pt}}c@{\hspace{3pt}}c@{\hspace{3pt}}c@{\hspace{3pt}}c@{\hspace{0pt}}}
       \includegraphics[width=0.16\columnwidth]{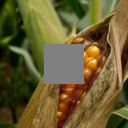} &  
       \includegraphics[width=0.16\columnwidth]{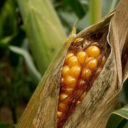} &
       \includegraphics[width=0.16\columnwidth]{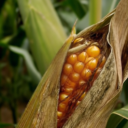} & 
       \includegraphics[width=0.16\columnwidth]{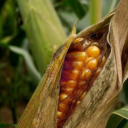} &
       \includegraphics[width=0.16\columnwidth]{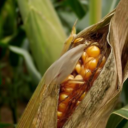} &
       \includegraphics[width=0.16\columnwidth]{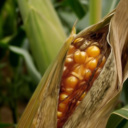}  \\
       \includegraphics[width=0.16\columnwidth]{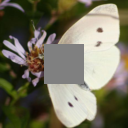} &  
       \includegraphics[width=0.16\columnwidth]{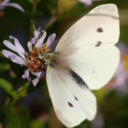} &
       \includegraphics[width=0.16\columnwidth]{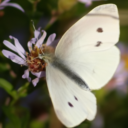} & 
       \includegraphics[width=0.16\columnwidth]{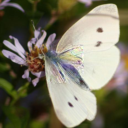} &
       \includegraphics[width=0.16\columnwidth]{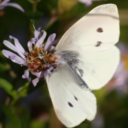} &
       \includegraphics[width=0.16\columnwidth]{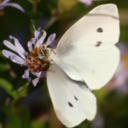}  \\
       \includegraphics[width=0.16\columnwidth]{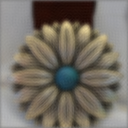} &  
       \includegraphics[width=0.16\columnwidth]{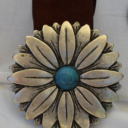} &
       \includegraphics[width=0.16\columnwidth]{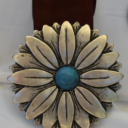} & 
       \includegraphics[width=0.16\columnwidth]{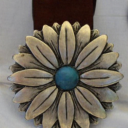} &
       \includegraphics[width=0.16\columnwidth]{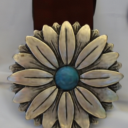} &
       \includegraphics[width=0.16\columnwidth]{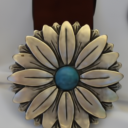}  \\
       \includegraphics[width=0.16\columnwidth]{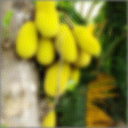} &  
       \includegraphics[width=0.16\columnwidth]{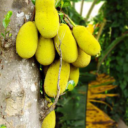} &
       \includegraphics[width=0.16\columnwidth]{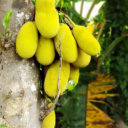} & 
       \includegraphics[width=0.16\columnwidth]{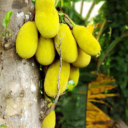} &
       \includegraphics[width=0.16\columnwidth]{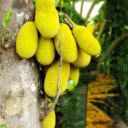} &
       \includegraphics[width=0.16\columnwidth]{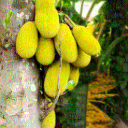}  \\
       \includegraphics[width=0.16\columnwidth]{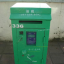} &  
       \includegraphics[width=0.16\columnwidth]{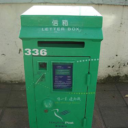} &
       \includegraphics[width=0.16\columnwidth]{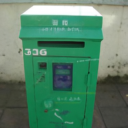} & 
       \includegraphics[width=0.16\columnwidth]{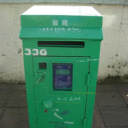} &
       \includegraphics[width=0.16\columnwidth]{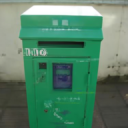} &
       \includegraphics[width=0.16\columnwidth]{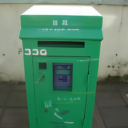}  \\
       \includegraphics[width=0.16\columnwidth]{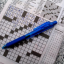} &  
       \includegraphics[width=0.16\columnwidth]{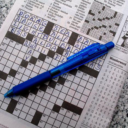} &
       \includegraphics[width=0.16\columnwidth]{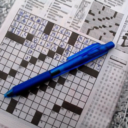} & 
       \includegraphics[width=0.16\columnwidth]{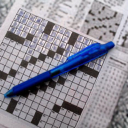} &
       \includegraphics[width=0.16\columnwidth]{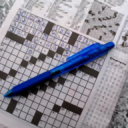} &
       \includegraphics[width=0.16\columnwidth]{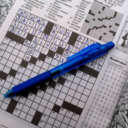}  \\
       \scriptsize Distorted & \scriptsize Ground Truth & \scriptsize Ours & \scriptsize RED-Diff & \scriptsize  OT-ODE & \scriptsize $\Pi$GDM \vspace{-10pt}
  \end{tabular}    
  \end{center}
  \caption{Qualitative comparison for linear inverse problems on ImageNet-128 for the noiseless case. GT samples come from the face-blurred ImageNet-128 validation set.}\label{afig:linear_inv_imagenet}
\end{figure}

\begin{figure}
  \begin{center}
  \begin{tabular}{@{\hspace{0pt}}c@{\hspace{3pt}}c@{\hspace{3pt}}c@{\hspace{3pt}}c@{\hspace{3pt}}c@{\hspace{3pt}}c@{\hspace{0pt}}}
       \includegraphics[width=0.16\columnwidth]{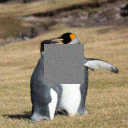} &  
       \includegraphics[width=0.16\columnwidth]{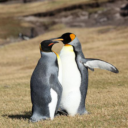} &
       \includegraphics[width=0.16\columnwidth]{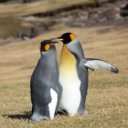} & 
       \includegraphics[width=0.16\columnwidth]{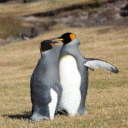} &
       \includegraphics[width=0.16\columnwidth]{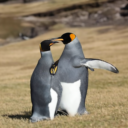} &
       \includegraphics[width=0.16\columnwidth]{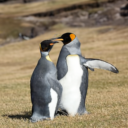}  \\
       \includegraphics[width=0.16\columnwidth]{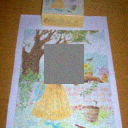} &  
       \includegraphics[width=0.16\columnwidth]{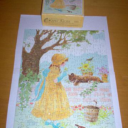} &
       \includegraphics[width=0.16\columnwidth]{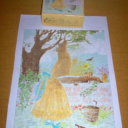} & 
       \includegraphics[width=0.16\columnwidth]{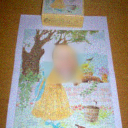} &
       \includegraphics[width=0.16\columnwidth]{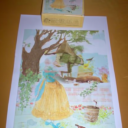} &
       \includegraphics[width=0.16\columnwidth]{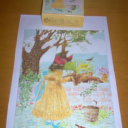}  \\
       \includegraphics[width=0.16\columnwidth]{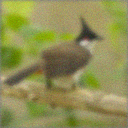} &  
       \includegraphics[width=0.16\columnwidth]{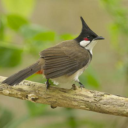} &
       \includegraphics[width=0.16\columnwidth]{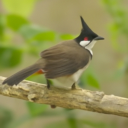} & 
       \includegraphics[width=0.16\columnwidth]{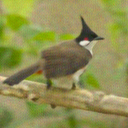} &
       \includegraphics[width=0.16\columnwidth]{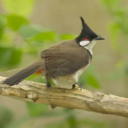} &
       \includegraphics[width=0.16\columnwidth]{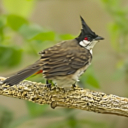}  \\
       \includegraphics[width=0.16\columnwidth]{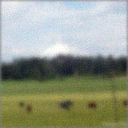} &  
       \includegraphics[width=0.16\columnwidth]{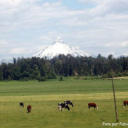} &
       \includegraphics[width=0.16\columnwidth]{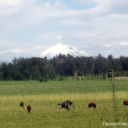} & 
       \includegraphics[width=0.16\columnwidth]{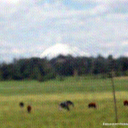} &
       \includegraphics[width=0.16\columnwidth]{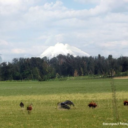} &
       \includegraphics[width=0.16\columnwidth]{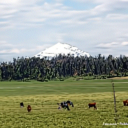}  \\
       \includegraphics[width=0.16\columnwidth]{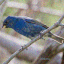} &  
       \includegraphics[width=0.16\columnwidth]{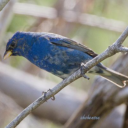} &
       \includegraphics[width=0.16\columnwidth]{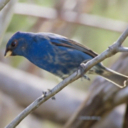} & 
       \includegraphics[width=0.16\columnwidth]{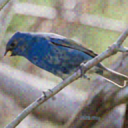} &
       \includegraphics[width=0.16\columnwidth]{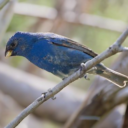} &
       \includegraphics[width=0.16\columnwidth]{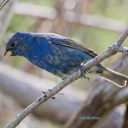}  \\
       \includegraphics[width=0.16\columnwidth]{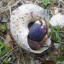} &  
       \includegraphics[width=0.16\columnwidth]{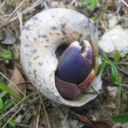} &
       \includegraphics[width=0.16\columnwidth]{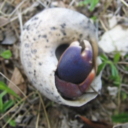} & 
       \includegraphics[width=0.16\columnwidth]{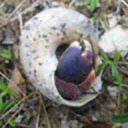} &
       \includegraphics[width=0.16\columnwidth]{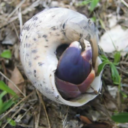} &
       \includegraphics[width=0.16\columnwidth]{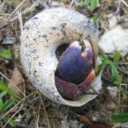}  \\
       \scriptsize Distorted & \scriptsize Ground Truth & \scriptsize Ours & \scriptsize RED-Diff & \scriptsize  OT-ODE & \scriptsize $\Pi$GDM \vspace{-10pt}
  \end{tabular}    
  \end{center}
  \caption{Qualitative comparison for linear inverse problems on ImageNet-128 for the noisy case. GT samples come from the face-blurred ImageNet-128 validation set.}\label{afig:linear_inv_imagenet_2}
\end{figure}

\begin{figure}
  \begin{center}
  \begin{tabular}{@{\hspace{0pt}}c@{\hspace{3pt}}c@{\hspace{3pt}}c@{\hspace{3pt}}c@{\hspace{0pt}}}
       \includegraphics[width=0.22\columnwidth]{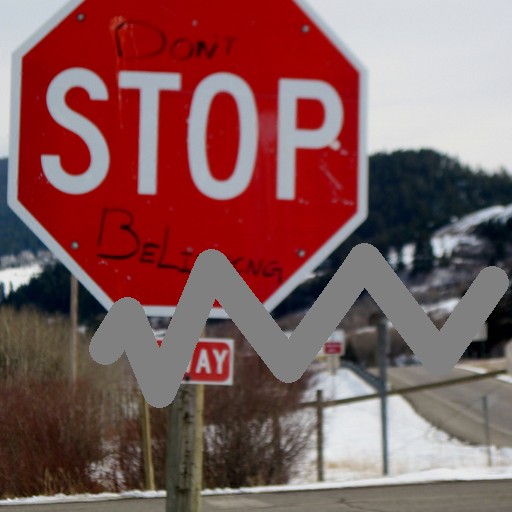} &  
       \includegraphics[width=0.22\columnwidth]{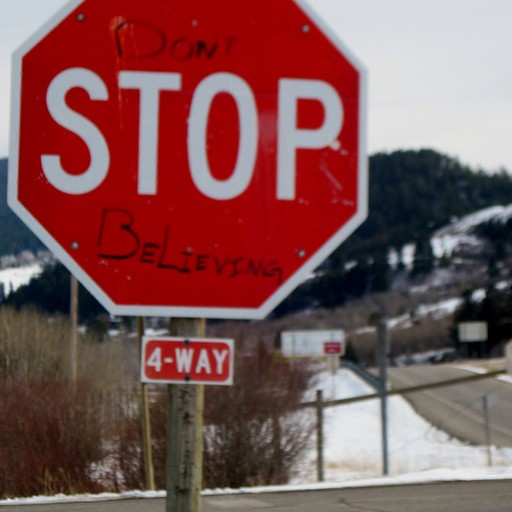} &
       \includegraphics[width=0.22\columnwidth]{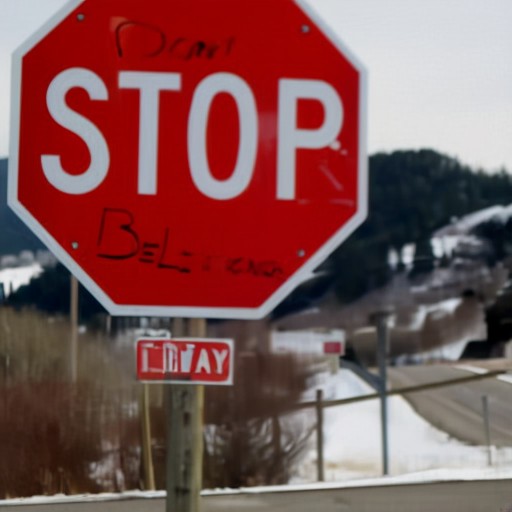} & 
       \includegraphics[width=0.22\columnwidth]{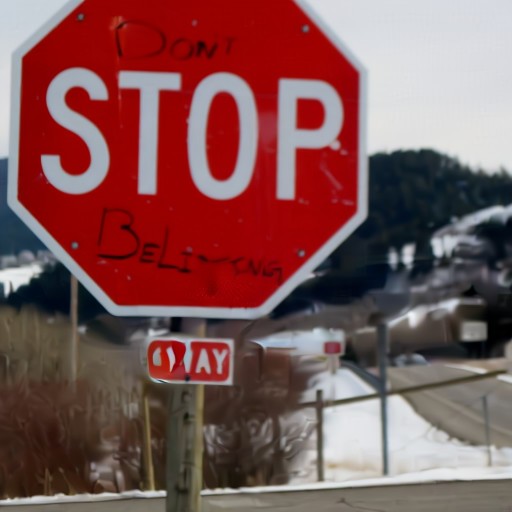}   \\
       \includegraphics[width=0.22\columnwidth]{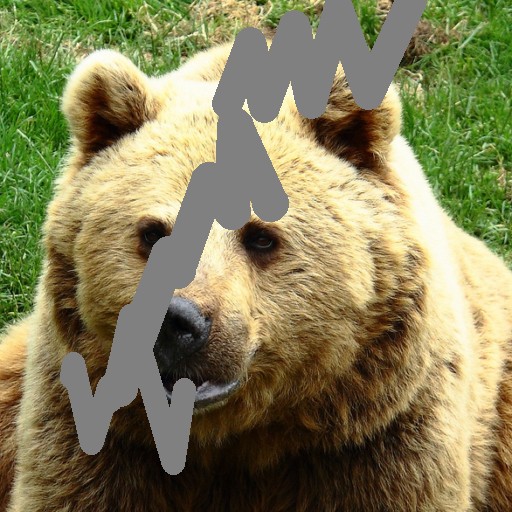} &  
       \includegraphics[width=0.22\columnwidth]{figures/t2i_jpg/svd_inpainting_brush_stroke_3513_gt.jpg} &
       \includegraphics[width=0.22\columnwidth]{figures/t2i_jpg/svd_inpainting_brush_stroke_3513_ours.jpg} & 
       \includegraphics[width=0.22\columnwidth]{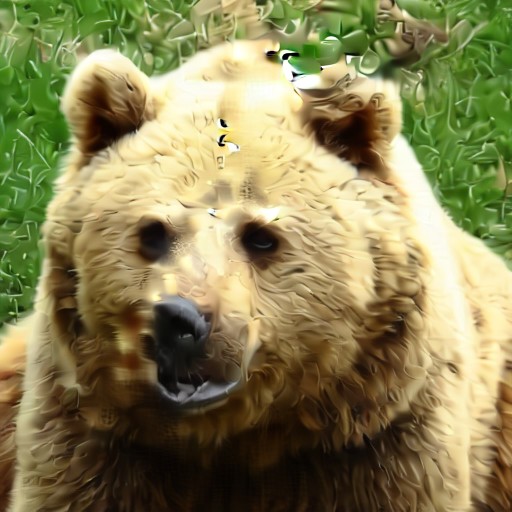}   \\
       \includegraphics[width=0.22\columnwidth]{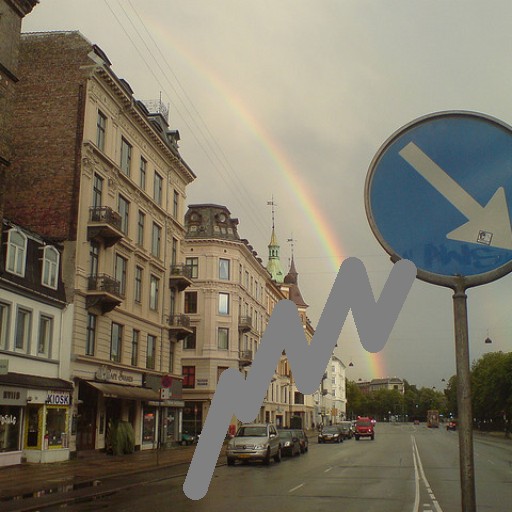} &  
       \includegraphics[width=0.22\columnwidth]{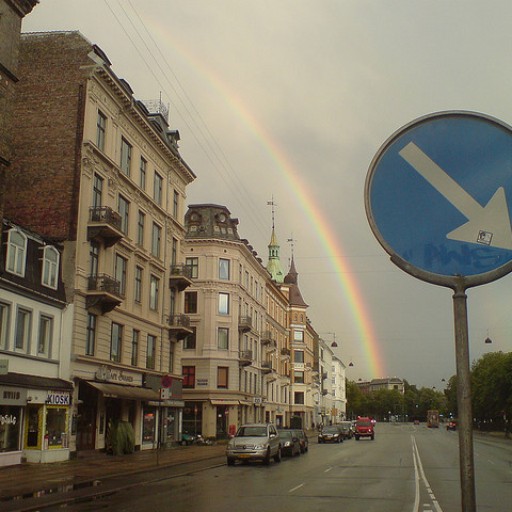} &
       \includegraphics[width=0.22\columnwidth]{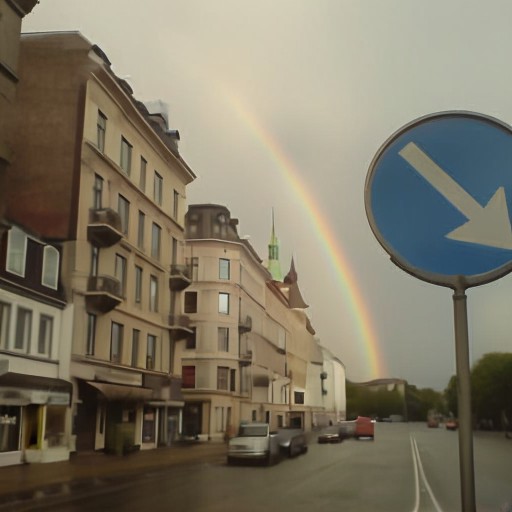} & 
       \includegraphics[width=0.22\columnwidth]{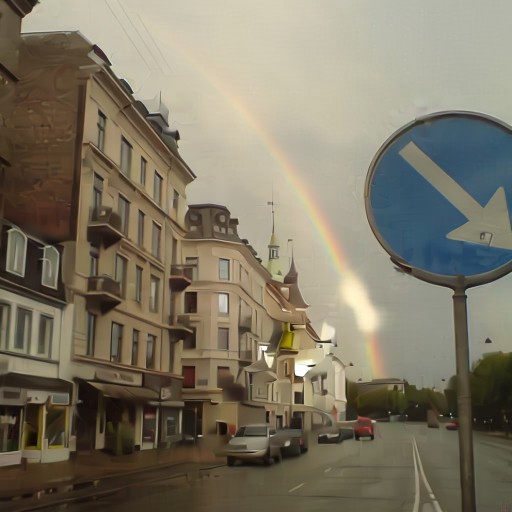}   \\
       \includegraphics[width=0.22\columnwidth]{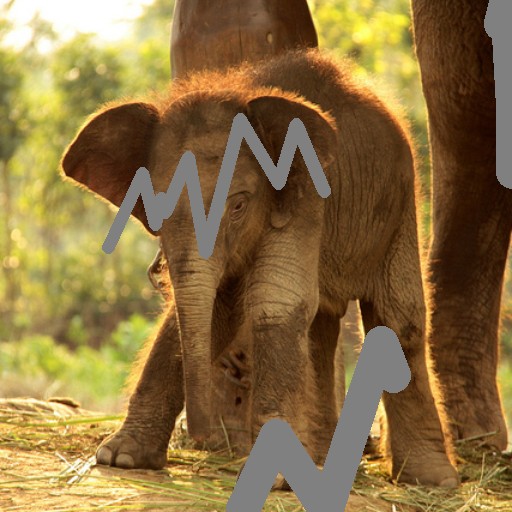} &  
       \includegraphics[width=0.22\columnwidth]{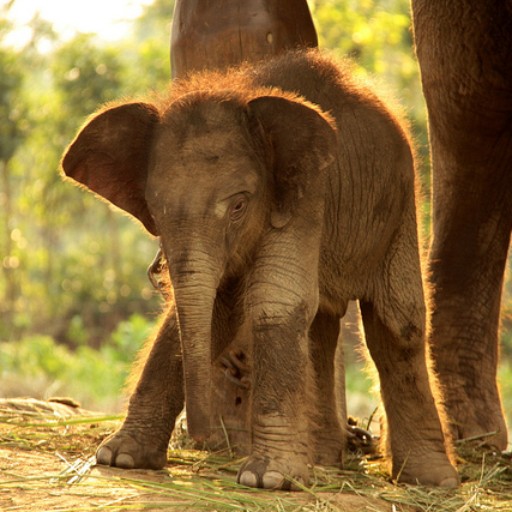} &
       \includegraphics[width=0.22\columnwidth]{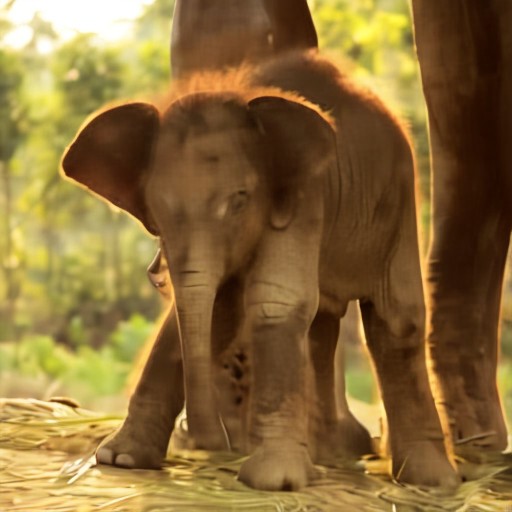} & 
       \includegraphics[width=0.22\columnwidth]{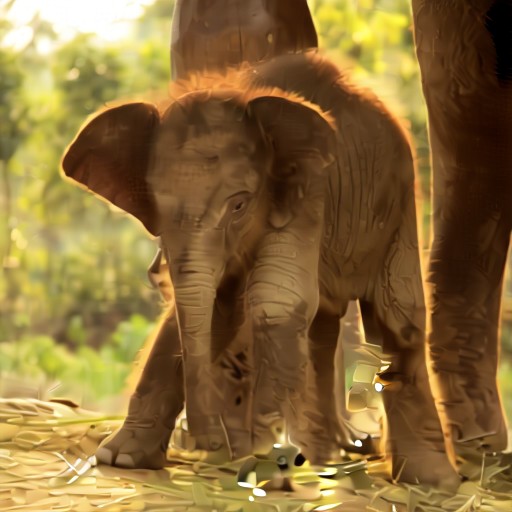}   \\
       \includegraphics[width=0.22\columnwidth]{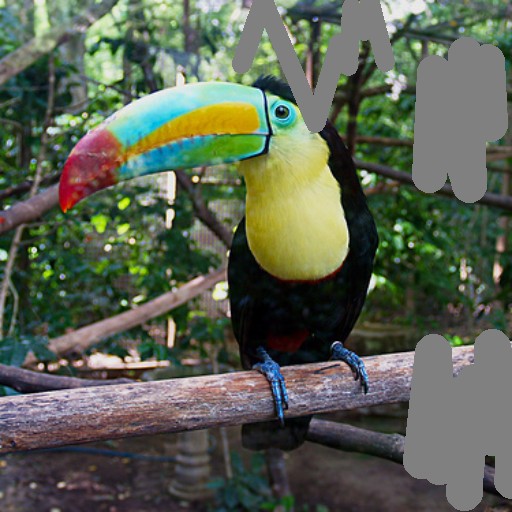} &  
       \includegraphics[width=0.22\columnwidth]{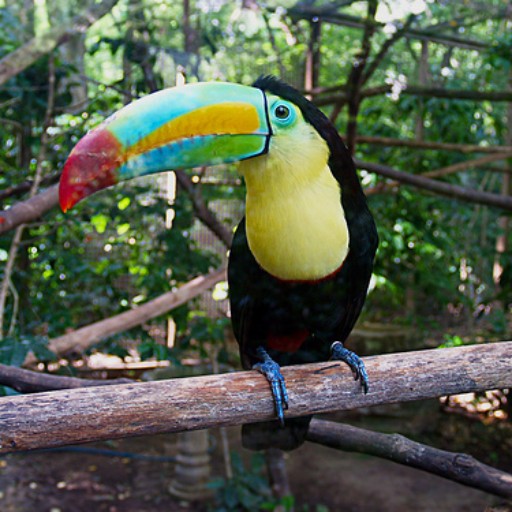} &
       \includegraphics[width=0.22\columnwidth]{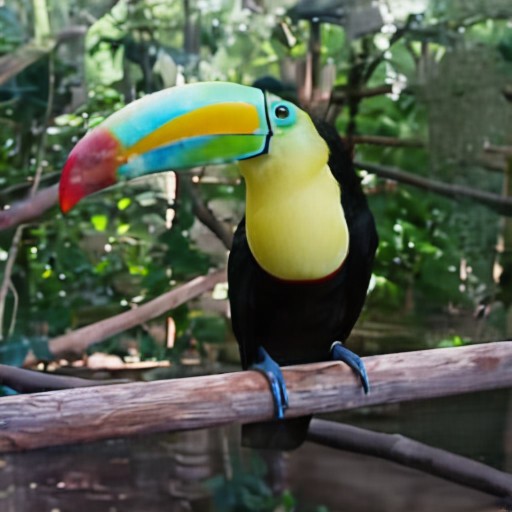} & 
       \includegraphics[width=0.22\columnwidth]{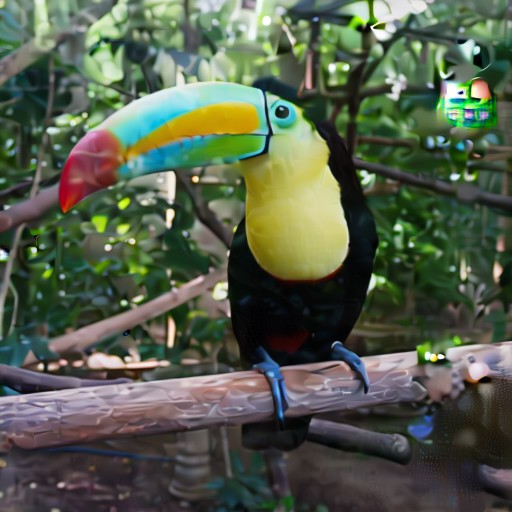}   \\
       
       \small Distorted & \small Ground Truth & \small Ours & \small RED-Diff \vspace{-10pt}
  \end{tabular}    
  \end{center}
  \caption{Qualitative comparison for free-form inpainting on the MS-COCO dataset using a T2I latent FM model. GT samples come from the MS-COCO validation set.}\label{afig:inp_latent_image}
\end{figure}

\end{document}

%% file: math_commands.tex
\usepackage{xcolor}

\newcommand{\norm}[1]{\left\Vert#1\right\Vert}

\newcommand{\abs}[1]{\left\vert#1\right\vert}

\newcommand{\parr}[1]{\left (#1\right )}
\newcommand{\brac}[1]{\left [#1\right ]}

\newcommand{\Real}{\mathbb R}

\newcommand{\eps}{\varepsilon}
\newcommand{\too}{\rightarrow}

\newcommand{\divv}{\mathrm{div}} 

\definecolor{mygray}{gray}{0.95}

\newcommand{\eg}{{e.g.}}
\newcommand{\ie}{{i.e.}}

\usepackage{amsmath,amsfonts,bm}

\def\eqref#1{equation~\ref{#1}}

\def\1{\bm{1}}

\def\eps{{\epsilon}}

\DeclareMathAlphabet{\mathsfit}{\encodingdefault}{\sfdefault}{m}{sl}
\SetMathAlphabet{\mathsfit}{bold}{\encodingdefault}{\sfdefault}{bx}{n}

\def\gF{{\mathcal{F}}}

\def\gL{{\mathcal{L}}}

\def\gN{{\mathcal{N}}}
\def\gO{{\mathcal{O}}}

\def\gR{{\mathcal{R}}}

\def\gT{{\mathcal{T}}}

\newcommand{\E}{\mathbb{E}}

\newcommand{\Var}{\mathrm{Var}}

\usepackage{tabularx,ragged2e,booktabs,caption}
\newcolumntype{C}[1]{>{\Centering}m{#1}}

\newcolumntype{Z}[1]{>{\Left}m{#1}}